%% file: main.tex
\documentclass[10pt,journal,compsoc]{IEEEtran}

\usepackage{enumitem}
\usepackage{xspace}
\usepackage{pgfplots}

\ifCLASSOPTIONcompsoc
  \usepackage[caption=false,font=normalsize,labelfont=sf,textfont=sf]{subfig}
\else
  \usepackage[caption=false,font=footnotesize]{subfig}
\fi

\usepackage{graphicx}
\usepackage{amsthm}
\usepackage{amsmath, amssymb}
\usepackage{url}
\usepackage{booktabs}

\theoremstyle{plain}
\newtheorem{theorem}{Theorem}[section]
\newtheorem{lemma}[theorem]{Lemma}
\newtheorem{corollary}[theorem]{Corollary}
\newtheorem{assumption}[theorem]{Assumption}
\theoremstyle{definition}
\newtheorem{definition}[theorem]{Definition}

\newcommand{\methodname}{\textit{Gecko}\xspace}

\pgfplotsset{compat=1.18}

\begin{document}

\bstctlcite{IEEE:BSTcontrol}

\title{Gecko: Fast Private Inference via Secure Public Encoder Offloading}

\author{Cheng'an Wei, Kai Chen, Yue Zhao, Congyi Li, and Shenchen Zhu%
\IEEEcompsocitemizethanks{%
\IEEEcompsocthanksitem The authors are with the Institute of Information
Engineering, Chinese Academy of Sciences, China, and the School of Cyber
Security, University of Chinese Academy of Sciences, China
(e-mail: weichengan@iie.ac.cn; chenkai@iie.ac.cn; zhaoyue@iie.ac.cn;
licongyi@iie.ac.cn; zhushenchen@iie.ac.cn). Corresponding author: Kai Chen.}}

\IEEEtitleabstractindextext{%
\begin{abstract}

Private inference protects both user inputs and server models during neural network inference, but existing solutions remain too slow for practical deployment. This motivates recent efforts to run a public encoder, such as a pretrained backbone, outside the protection boundary and evaluate only a small private predictor cryptographically. While appealing for efficiency, this design is not inherently secure: naively offloading a public encoder may create a \emph{feature-space shortcut}: an extraction adversary may learn the remaining private predictor's feature-to-output mapping more easily than the original model's input-to-output behavior.

We present \methodname, designed to limit this additional risk while retaining a compact encrypted predictor. We leverage a frozen backbone that contributes hierarchical features, fixed Fastfood projections that compress them, and private feature gating that prepares them for prediction. We formalize ideal independence and information-preservation conditions as design guidance, then separately evaluate component-reuse extraction attacks. Across image and audio tasks, \methodname achieves 0.4--2.2-second inference with at most 10.8\,MB communication and accuracy comparable to transfer-learning baselines. Under the evaluated attacks, reusing the offloaded public encoder provides no significant advantage to model-extraction adversaries. Source code and a demo are available at \url{https://github.com/CassiniHuy/gecko-infer}.

\end{abstract}

}

\maketitle

\input{sections/intro}
\input{sections/problem}
\input{sections/overview}
\input{sections/theory}
\input{sections/method}
\input{sections/exps}
\input{sections/related}
\input{sections/discon}

\bibliographystyle{IEEEtran}
\bibliography{refs}

\end{document}

%% file: sections/intro.tex
\section{Introduction}

Deep learning inference is increasingly offered as a remote service for sensitive applications such as biomedical image classification~\cite{medmnistv1,medmnistv2}. Private inference aims to protect both parties in this setting: the server should not learn the user's input or output, and the user should not learn the server's proprietary model beyond the intended prediction. Fully Homomorphic Encryption (FHE) provides this confidentiality by allowing the server to compute directly on ciphertexts~\cite{cryptoNets16icml}.

The main obstacle is cost. Neural-network operations are far more expensive under encryption than in plaintext, and the feature-extraction backbone usually dominates modern inference. Even optimized systems still require about one minute for a ResNet50 image~\cite{ngSoKCryptographicNeuralNetwork2023,he2024RhombusFastHomomorphic,feng2025PantherPracticalSecure,liu2025AntelopeFastSecure}. Public-encoder offloading offers a different optimization: execute a generic pretrained encoder outside FHE and protect only a lightweight task-specific predictor~\cite{brutzkusLowLatencyPrivacy2019,ruanPrivateEfficientAccurate2023,xia2026CryptPEFTEfficientPrivate}. This follows a \emph{public-resource assumption}: the offloaded encoder is independently available and contains no proprietary task adaptation.

Consider a chest-X-ray service. The user can run a frozen public image encoder locally, but its representation is not a pneumonia diagnosis; the service-specific mapping remains in the server's proprietary predictor. The user therefore encrypts the public representation, the server evaluates the private predictor under FHE, and the user decrypts the returned result. Offloading eliminates the need to protect the expensive generic encoder without eliminating the server's proprietary contribution.

\noindent\textbf{Feature-space shortcut problem.}
Ordinary prediction access already entails model-extraction risk: an adversary can query class probabilities and train a transfer-learning surrogate~\cite{orekondy2019knockoff}. While we cannot eliminate this baseline risk, we should not introduce incremental risk through offloading. A naive split commonly exposes the exact final-layer representation consumed by the private predictor. Because deep features are compact and selective~\cite{zivOpeningBlackBox2017}, this interface may reveal a narrow task-aligned feature space and reduce extraction from learning the full input--output behavior to learning a simpler feature-to-output mapping. We call this additional simplification the \emph{feature-space shortcut}.

The basic remedy is to enrich the public representation, but doing so introduces two challenges. First, the public representation should be richer without becoming task-adapted or forcing a large private predictor. Second, representation richness needs a principled assessment: task accuracy alone indicates utility but not the feature richness the public representation has.

\noindent\textbf{Approach.}
We propose \methodname, a private-inference framework designed to limit the additional shortcut created by public encoding. A frozen public backbone contributes features from multiple abstraction levels, retaining low-, mid-, and high-level characteristics rather than only its selective final representation. Fixed Fastfood-based random projections~\cite{le2013FastfoodApproximatingKernel} compress each level, and a private hierarchical feature-gating layer balances the projected groups before being fused into a lightweight predictor. Only this task-specific predictor is evaluated under FHE.

We also formulate two ideal design criteria: \emph{prediction-artifact independence}, which excludes task-specific private artifacts from the public component, and \emph{information preservation}, which prevents the exposed interfaces from collapsing the inference behavior. These criteria characterize prediction-functional uncertainty, which, while not a direct measure of model extraction hardness, helps us gauge it. In practice, we use estimated mutual information as an offline, relative diagnostic of representation richness and then empirically evaluate extraction risk by measuring the difference in attack accuracy after offloading.

Across general-image, biomedical, and audio tasks~\cite{krizhevsky2009learning,deng2012mnist,parkhi2012cats,lecun1998gradient,helber2019eurosat,zohar_jackson_2018_1342401,medmnistv2}, \methodname completes inference in 0.4--2.2\,seconds with at most 10.8\,MB communication while maintaining accuracy comparable to transfer-learning baselines. Under our evaluated attacks, direct reuse of the public encoder and projection offered no consistent extraction advantage over free surrogate tuning. Our contributions are:

\begin{itemize}[leftmargin=*]
    \item We develop a private-inference partition that executes an independently public encoder outside FHE while retaining the task-specific predictor under FHE, substantially reducing protected computation.
    \item We construct a compact hierarchical representation using multi-level features, fixed Fastfood projections, and private feature gating, so the encrypted predictor remains lightweight without relying on a selective final-layer interface.
    \item We formalize ideal information-theoretic criteria for this partition and evaluate its efficiency, utility, and incremental extraction risk across realistic tasks and component-reuse attack settings.
\end{itemize}

The source code is available at \url{https://github.com/CassiniHuy/gecko-infer}.

%% file: sections/problem.tex
\section{Problem Setting and Challenges}\label{sec:problem-setting}

\subsection{Private Inference with Public Encoders}\label{subsec:pi-setting}

We consider a client--server classification service. The user holds a private input, and the server holds a proprietary task-specific predictor. Conventional private inference evaluates the complete model on encrypted inputs, but protecting a large generic feature extractor dominates the cost. As Figure~\ref{fig:problem-setting} illustrates, public-encoder offloading instead lets the user execute an independently available encoder in plaintext, encrypt its representation, and send only that ciphertext to the proprietary predictor. The encoder cannot produce the service output by itself: it supplies generic features but contains no private task-specific mapping.

This partition changes the model-protection boundary. The public encoder needs no model-side confidentiality, whereas the predictor weights and internal state remain private. Our question is therefore whether exposing the encoder and its representation gives an adversary \emph{additional} extraction leverage beyond normal prediction access.

\subsection{Threat Model}\label{subsec:threat-model}

Following prior private-inference work~\cite{usenix22cheetah,kimOptimizedPrivacyPreservingCNN2023,brutzkusLowLatencyPrivacy2019,mishraDelphiCryptographicInference2020,singhHyenaBalancingPacking2024,yang2025RBOOTAcceleratingHomomorphic}, we assume two semi-honest parties without a trusted third party. Both follow the prescribed protocol while attempting to infer additional information.

\begin{itemize}[leftmargin=*]
    \item The \textit{curious server} receives only encrypted representations and returns encrypted predictions. It does not hold the decryption key and should learn neither the user's input nor output; this protection follows from the underlying FHE scheme.
    \item The \textit{model-extraction adversary} controls the user side and may choose valid raw inputs, submit their prescribed encodings, decrypt the returned results, and train a surrogate from the intended class-probability outputs. We assume it knows the protocol, predictor architecture, public backbone, projection, and all other public components, but not the private predictor parameters or server state.
\end{itemize}

Because the setting is semi-honest, every submitted representation is assumed to be generated by the prescribed public encoder from a valid raw input. Chosen-latent queries, malformed ciphertexts, and other protocol deviations are outside scope. Likewise, the model-extraction adversary's access to class probabilities is an intended API disclosure, not a breach.

\noindent\textit{Security goal: avoiding additional model-side exposure.}
Any prediction API may be replicated through ordinary black-box queries. Our goal is to ensure that \methodname's offloaded public components do not create an easier feature-space route than that baseline without offloading. We assess this incremental risk under the specific attacks in Section~\ref{subsec:sec-eval}.

\begin{figure}[bt]
    \centering
    \includegraphics[width=0.98\linewidth]{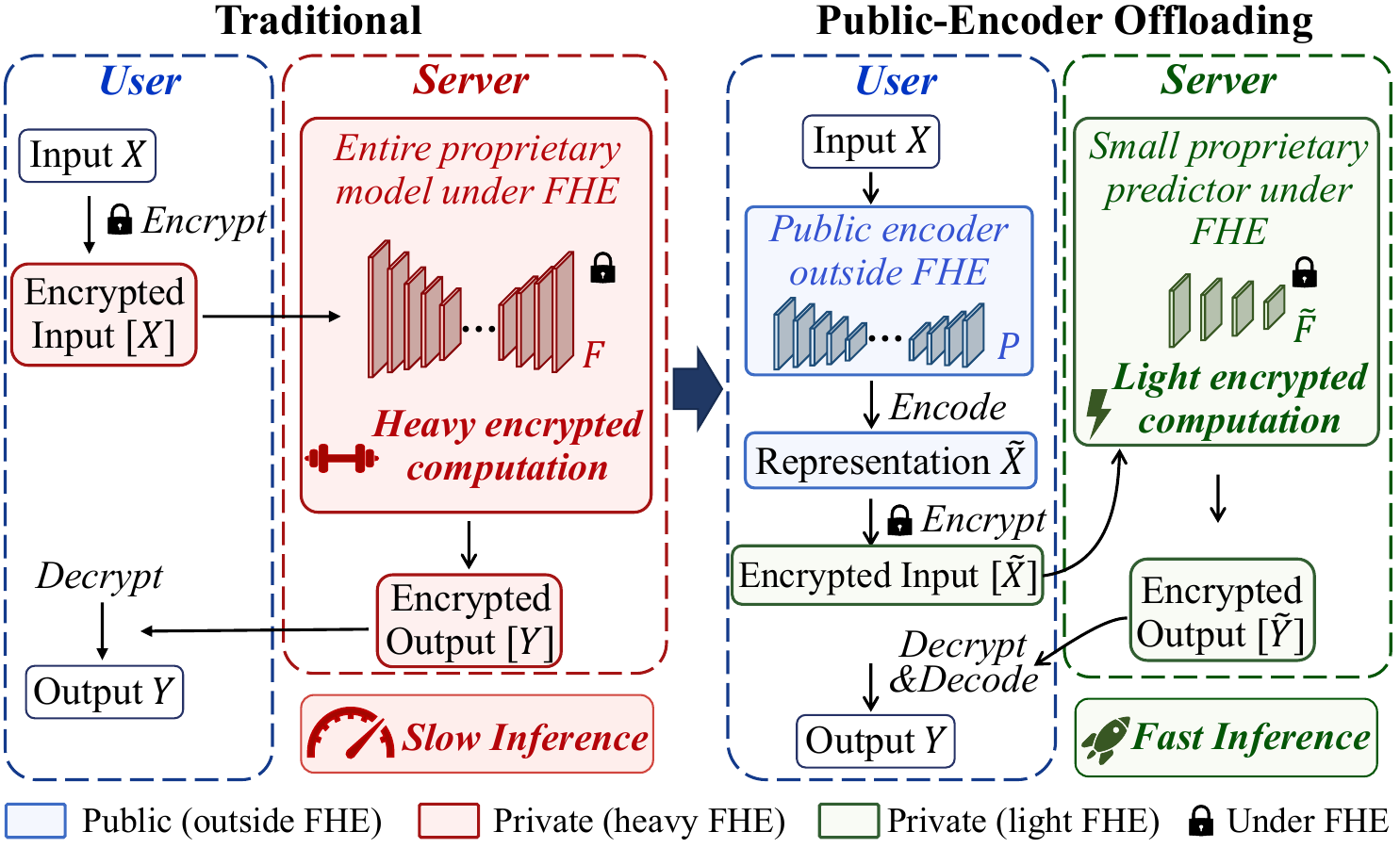}
    \caption{Traditional private inference protects the complete model under FHE. Public-encoder offloading executes an independently public encoder outside FHE and protects the task-specific predictor. The user still needs the predictor to obtain $Y$; \methodname is designed to limit additional model-side exposure caused by this partition.}
    \label{fig:problem-setting}
\end{figure}

\subsection{Risks of Naive Public-Encoder Offloading}\label{sec:why-fail}

\noindent\textbf{Exposing proprietary artifacts.}
One approach moves part of a proprietary model outside protection, exposing partial weights~\cite{mo2020DarkneTZModelPrivacy,shen2022SOTERGuardingBlackbox}, selected neurons~\cite{renPrivDNNSecureMultiParty2024}, intermediate activations~\cite{chenTHEXPrivacyPreservingTransformer2022}, or activation statistics~\cite{yan2025CometAcceleratingPrivate,scope2025ccs}. These signals come from the private model and may directly aid surrogate training. \methodname instead offloads an independently available pretrained encoder (e.g., ResNet~\cite{he2016deep}) and protects the entire task-specific predictor. This public-resource assumption prevents direct disclosure of proprietary artifacts, but it does not by itself prevent the exposed representation from simplifying extraction.

\noindent\textbf{Feature-space shortcut.}
Prior public-encoder designs commonly place a shallow private predictor on the final backbone representation~\cite{brutzkusLowLatencyPrivacy2019,ruanPrivateEfficientAccurate2023}. As Figure~\ref{fig:failure-encoder} illustrates, this exact interface may expose a compact feature space that has already discarded much input variation and retained prediction-relevant cues. An adversary can still ignore this interface and conduct an ordinary raw-input attack; the incremental concern is that reusing it may reduce the search to a simpler mapping from the exposed features to outputs. This is the \emph{feature-space shortcut}.

\begin{figure}[tb]
    \centering
    \includegraphics[width=0.98\linewidth]{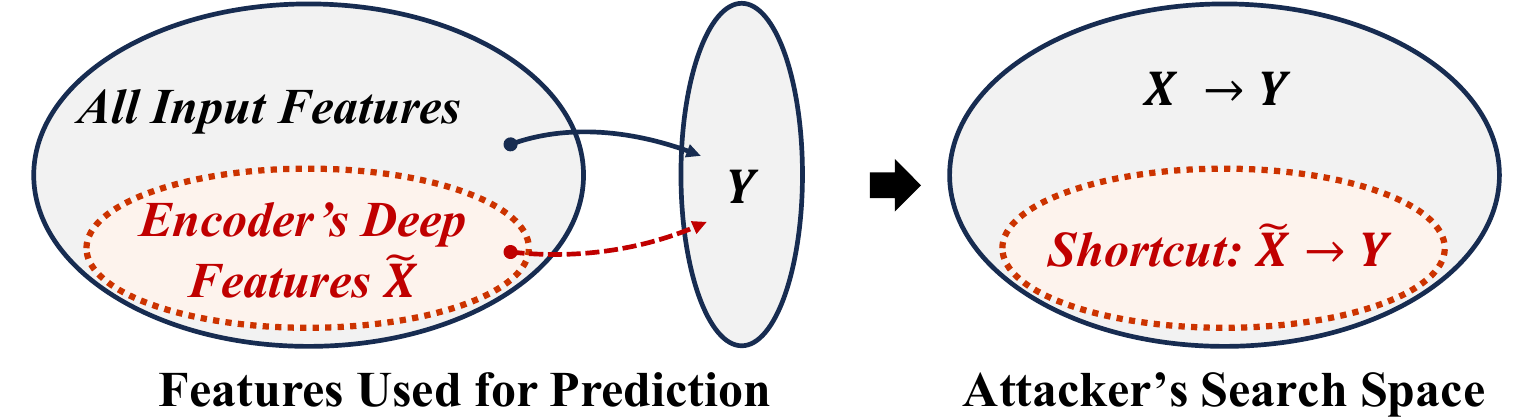}
    \caption{A feature-space shortcut. Naive final-layer offloading may expose a selective representation $\tilde{X}$, giving an extraction adversary an additional feature-to-output route beyond ordinary learning of $X\to Y$.}
    \label{fig:failure-encoder}
\end{figure}

Section~\ref{sec:criteria-formalization} characterizes the corresponding reduction in prediction-functional uncertainty, motivating the empirical extraction analysis. Section~\ref{subsec:sec-ablation} then compares this naive interface with \methodname at similar representation dimensions. The design challenge of \methodname is therefore to preserve richer input information without adapting the public encoder to the task or making the encrypted predictor too costly.

%% file: sections/overview.tex
\section{Gecko Overview and Design Principles}
\label{sec:gecko-overview}
\label{sec:criteria-formalization}

\subsection{Workflow Overview}

\methodname combines three design choices. First, a frozen, task-agnostic backbone supplies features from multiple levels rather than only its selective final layer. Second, fixed Fastfood projections compress those public features, while Hierarchical Feature Gating (HFG) privately balances their contributions without adding FHE depth after fusion. Third, ideal information-theoretic criteria guide the public/private partition, and estimated mutual information provides an offline diagnostic of input-representation richness.

\begin{figure}[tb]
    \centering
    \includegraphics[width=0.98\linewidth]{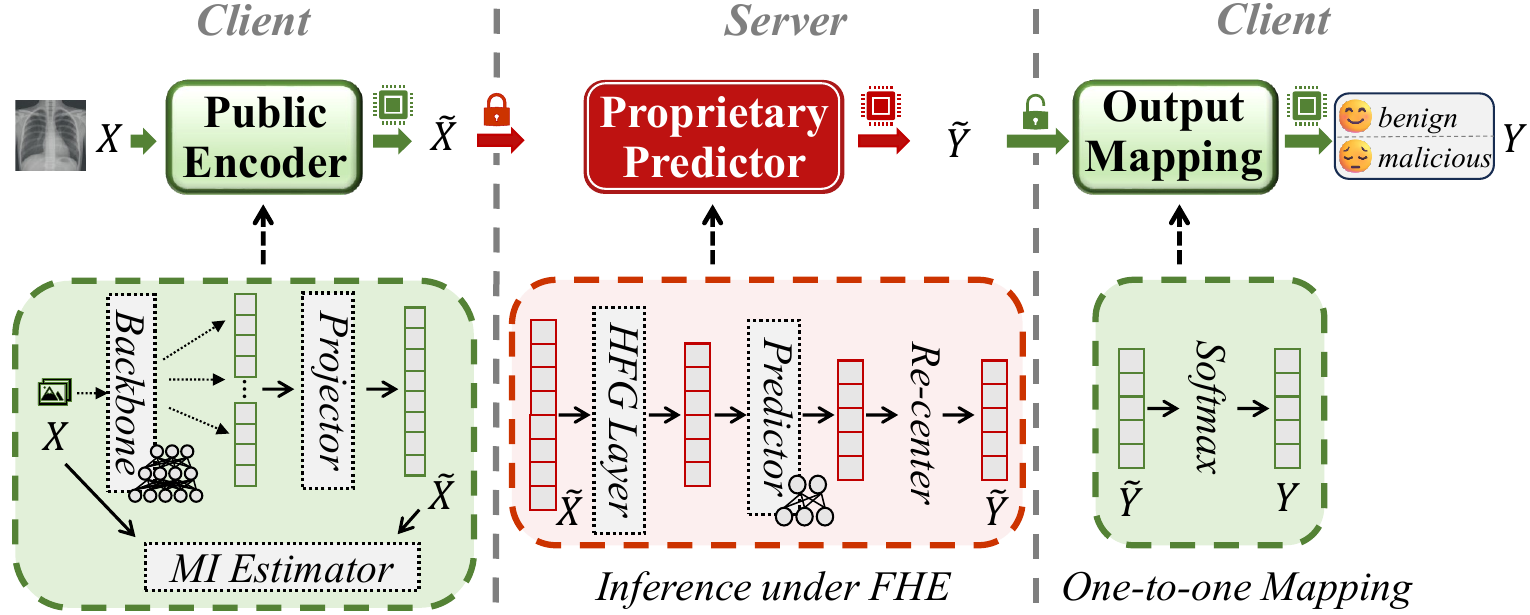}
    \caption{\methodname workflow. Online inference follows $X\to$ frozen multi-level backbone and fixed projections $\to\tilde{X}\to$ encryption $\to$ proprietary predictor under FHE $\to$ centered logits $\tilde{Y}\to$ decryption and softmax $\to Y$. HFG is private and fused into the predictor; the MI estimator runs offline during representation design.}
    \label{fig:gecko-streamline}
\end{figure}

\noindent\textbf{End-to-end inference.}
Figure~\ref{fig:gecko-streamline} separates the online workflow into three stages:
\begin{itemize}[leftmargin=*]
    \item \textit{Public encoding.} The user applies the frozen multi-level backbone and fixed projections in plaintext to obtain $\tilde{X}$.
    \item \textit{Proprietary prediction.} The user encrypts $\tilde{X}$, and the server evaluates the task-specific predictor under FHE to produce encrypted centered logits $\tilde{Y}$.
    \item \textit{Output reconstruction.} The user decrypts $\tilde{Y}$ and applies the fixed softmax mapping to recover the intended probabilities $Y$.
\end{itemize}

The public encoder contains no task-adapted parameters, whereas HFG and the predictor are learned for the service and remain private. The following formalization develops ideal conditions for this separation; Sections~\ref{subsec:sec-eval} and~\ref{subsec:sec-ablation} separately evaluate whether the public components help the tested extraction attacks.

%% file: sections/theory.tex
\subsection{Formal Setup}\label{subsec:problem-formulation}

We formalize the public component $P$ as side information and characterize when observing it preserves uncertainty about prediction-relevant functionality. The resulting criteria guide \methodname's public/private partition, while the extraction experiments evaluate the corresponding behavior under concrete learning attacks.

Let $F$ be a random prediction functionality from task input $X$ to output $Y$, and let $P$ be public side information. Raw inputs are sampled independently of the deployed functionality and public component, $X\perp(F,P)$, and $Y$ is determined by the prediction channel once $(X,F)$ are given. \methodname decomposes this workflow as $X\rightarrow\tilde{X}\rightarrow\tilde{Y}\rightarrow Y$, where $\tilde{F}$ denotes the protected mapping $\tilde{X}\to\tilde{Y}$. Table~\ref{tab:notations} summarizes the notation.

\input{data/notations}

\begin{definition}[\textit{Prediction-relevant functionality}]
\label{def:target-model}
Let $G$ be a random functionality with distribution $\mu$ over $\operatorname{supp}(G)$. For input $U$ and output $V$, it induces
\[
\Pr_{\mu}(v\mid u)=\sum_{g\in\operatorname{supp}(G)}
\Pr(v\mid u,g)\mu(g).
\]
We call $G$ prediction-relevant when the map $\mu\mapsto\Pr_{\mu}(\cdot\mid\cdot)$ is injective. Functionally inert parameter differences that induce the same prediction behavior are therefore treated as equivalent.
\end{definition}

\begin{lemma}[\textit{Relevance Equivalence}]
\label{lemma:relevance-equal}
Let $G$ be prediction-relevant for input $U$ and output $V$. Assume
$U\perp G$, $U\perp G\mid P$, and $V\perp P\mid(U,G)$. Then
\[
I(G;P)=0 \quad\Longleftrightarrow\quad I(V;P\mid U)=0.
\]
\end{lemma}
\begin{proof}
For each $p$ in the positive-probability support, the two input-independence assumptions and prediction-channel assumption give
\begin{align*}
\Pr(v\mid u,p)&=\sum_g\Pr(v\mid u,g)\Pr(g\mid p),\\
\Pr(v\mid u)&=\sum_g\Pr(v\mid u,g)\Pr(g).
\end{align*}
If $I(G;P)=0$, the mixing distributions are equal, implying
$I(V;P\mid U)=0$. Conversely, conditional independence of $V$ and $P$
given $U$ makes the two induced prediction functions equal. Injectivity in
Definition~\ref{def:target-model} then gives
$\Pr(G\mid P=p)=\Pr(G)$ and hence $I(G;P)=0$.
\end{proof}

For $(U,G,V)=(X,F,Y)$, the assumptions follow from $X\perp(F,P)$ and the prediction channel. Application to the latent workflow requires the corresponding latent-input assumptions stated below.

\subsection{Public-Offloading Criteria}
\label{subsec:criteria-formal}

We first exclude prediction-related artifacts from the public component, then characterize the complementary risk created by a selective latent interface.

\begin{definition}[\textit{Criterion 1: Prediction-Artifact Independence}]
\label{def:independence}
$P$ satisfies prediction-artifact independence if
\[
I(\tilde{Y};P\mid\tilde{X})=0.
\]
\end{definition}

Under the following latent-input conditions, Criterion~1 also characterizes the dependence between $P$ and the protected functionality $\tilde{F}$.

\begin{corollary}
\label{cor:latent-uncertainty}
Assume $\tilde{F}$ is prediction-relevant,
$\tilde{X}\perp\tilde{F}$,
$\tilde{X}\perp\tilde{F}\mid P$, and
$\tilde{Y}\perp P\mid(\tilde{X},\tilde{F})$.
If Criterion~1 holds, then $I(\tilde{F};P)=0$.
\end{corollary}
\begin{proof}
Apply Lemma~\ref{lemma:relevance-equal} with
$(U,G,V)=(\tilde{X},\tilde{F},\tilde{Y})$.
\end{proof}

These assumptions reflect \methodname's public-resource construction. Raw inputs are sampled independently of the deployed task-specific functionality, and the frozen backbone and fixed projection are selected independently of that functionality. Applying this fixed public encoder to independently sampled inputs yields the latent-input independence conditions used above.

\begin{definition}[\textit{Feature-Space Shortcut}]
\label{def:feature-shortcut}
For nonzero $H(X)$ and $H(Y)$, define the normalized information losses
\[
\begin{aligned}
\ell_X&=1-\frac{I(X;\tilde{X})}{H(X)}
=\frac{H(X\mid\tilde{X})}{H(X)},\\
\ell_Y&=1-\frac{I(Y;\tilde{Y})}{H(Y)}
=\frac{H(Y\mid\tilde{Y})}{H(Y)}.
\end{aligned}
\]
The pair $(\ell_X,\ell_Y)$ measures feature-space shortcut severity: larger values indicate that the latent interface preserves less uncertainty from the original input or output. A selective final-layer interface primarily creates a large input loss $\ell_X$.
\end{definition}

Criterion~1 removes prediction-related artifacts from $P$, while the feature-space shortcut captures the complementary risk that a selective representation presents the adversary with a substantially reduced input space.

\begin{lemma}[\textit{Conditional Functionality Correspondence}]
\label{lemma:feature-shortcut}
Suppose $F$ and $\tilde{F}$ are conditionally functionally equivalent once $P$ is fixed, meaning
\[
H(F\mid P,\tilde{F})=0,
\qquad
H(\tilde{F}\mid P,F)=0.
\]
Then
\[
H(F\mid P)=H(\tilde{F}\mid P).
\]
If Criterion~1 and the conditions of Corollary~\ref{cor:latent-uncertainty} hold, then
\[
H(F\mid P)=H(\tilde{F}).
\]
\end{lemma}
\begin{proof}
Expand $H(F,\tilde{F}\mid P)$ in both orders. Conditional functional equivalence gives
$H(F\mid P)=H(F,\tilde{F}\mid P)=H(\tilde{F}\mid P)$.
Criterion~1 and Corollary~\ref{cor:latent-uncertainty} give
$I(\tilde{F};P)=0$, and hence $H(\tilde{F}\mid P)=H(\tilde{F})$.
\end{proof}

\begin{definition}[\textit{Criterion 2: Information Preservation}]
\label{def:invariance}
Information preservation is satisfied for the task when $\ell_X=\ell_Y=0$, equivalently,
\[
I(X;\tilde{X})=H(X),\ I(Y;\tilde{Y})=H(Y).
\]
\end{definition}

\begin{assumption}[\textit{Lossless Representation--Functionality Consistency}]
\label{assump:representation-functionality}
If Criterion~2 holds, then the latent and original functionalities have equal uncertainty:
\[
H(\tilde{F})=H(F).
\]
\end{assumption}

This assumption captures the limiting behavior behind representation preservation. As the latent interface retains more input and output information, it distinguishes a broader range of prediction-relevant behaviors, so the uncertainty of $\tilde{F}$ is expected to approach that of $F$. Assumption~\ref{assump:representation-functionality} requires only the lossless endpoint of this intuition and leaves the rate of approach unrestricted.

\begin{theorem}[\textit{Ideal Public-Offloading Guarantee}]
\label{thm:uncertainty-lemma}
Assume the conditional functional-equivalence conditions of Lemma~\ref{lemma:feature-shortcut}, the conditions of Corollary~\ref{cor:latent-uncertainty}, Criteria~1--2, and Assumption~\ref{assump:representation-functionality}. Then
\[
H(F\mid P)=H(F).
\]
\end{theorem}
\begin{proof}
Lemma~\ref{lemma:feature-shortcut} and Criterion~1 give
$H(F\mid P)=H(\tilde{F})$. Criterion~2 and Assumption~\ref{assump:representation-functionality} give $H(\tilde{F})=H(F)$. Combining the two equalities yields $H(F\mid P)=H(F)$.
\end{proof}

The theorem places ideal public-encoder offloading at the same prediction-functional uncertainty as protecting the entire pipeline: exposing the public encoder leaves the original uncertainty $H(F)$ unchanged. For comparison with a traditional end-to-end trained model $F_{\mathrm{e2e}}$, we model function families with comparable predictive capacity and task performance as having comparable functionality uncertainty, $H(F)\approx H(F_{\mathrm{e2e}})$. Prior results show that predictors built from pretrained intermediate representations can match or approach conventional end-to-end fine-tuning across diverse transfer settings~\cite{evciHead2ToeUtilizingIntermediate2022,tu2023VisualQueryTuning,wang2025UnderstandingDeepRepresentation}, supporting this reference choice. Thus, at the ideal limit, \methodname targets the functionality uncertainty of full-model protection while moving generic public computation outside FHE.

\noindent\textbf{Implications for \methodname.}
Criterion~1 calls for a frozen public backbone and fixed projection free of task-adapted artifacts. Criterion~2 calls for a broad multi-level input representation and a bijective centered-logit output mapping. The centered-logit mapping realizes $\ell_Y=0$ exactly. The multi-level projected representation reduces $\ell_X$, and estimated NPMI provides relative evidence of progress toward the lossless endpoint. Section~\ref{sec:gecko-method} develops these choices and the offline mutual-information diagnostic.

%% file: data/notations.tex
\begin{table}[htb]
\centering
\caption{Summary of Notation.}
\label{tab:notations}
\begin{tabular}{cc}
\toprule
\textbf{Symbol} & \textbf{Description} \\
\midrule
$X,\, Y$ & Task input and output\\
$F$ & End-to-end functionality $X\mapsto Y$\\
$\tilde{X},\, \tilde{Y}$ & Input representation and protected output\\
$\tilde{F}$ & Protected private component of functionality $\tilde{X}\mapsto\tilde{Y}$\\
$P$ & Exposed public component or side information\\
$H(\cdot),\,H(\cdot\mid \cdot)$ & Entropy and conditional entropy\\
$I(\cdot;\cdot),\,I(\cdot;\cdot\mid \cdot)$ & Mutual and conditional mutual information\\
$A\perp B\mid C$ & Conditional independence\\
$\ell_X,\,\ell_Y$ & Normalized input and output information losses\\
\bottomrule
\end{tabular}
\end{table}

%% file: sections/method.tex
\section{Gecko Method}\label{sec:gecko-method}

Guided by the design principles in Section~\ref{sec:criteria-formalization}, we now instantiate \methodname's three workflow stages.

\subsection{Stage 1: Public encoding}\label{subsec:public-encoder-method}

The first stage constructs a useful public representation without adapting the encoder to the server's private task or data. Task adaptation would make the exposed component deployment-specific and undermine the public-resource assumption.

A natural starting point is therefore an off-the-shelf pretrained backbone. Such models are publicly available, broadly reusable, and widely used as generic feature extractors in transfer learning~\cite{rw2019timm,zhuang2021ComprehensiveSurveyTransfer}. We keep this backbone frozen. This choice is important not merely for convenience, but for security: any fine-tuning or task-specific adaptation would make the public encoder reflect the protected deployment, causing the exposed component to carry information about the private task or data distribution. 

\noindent\textbf{Hierarchical feature aggregation.} One may worry that such a restriction makes the encoder too weak. In practice, however, modern pretrained backbones already provide high-quality reusable representations. Transfer-learning studies have shown that useful predictive information is not limited to the final layer; intermediate activations can also provide comparable or even stronger utility for downstream tasks~\cite{evciHead2ToeUtilizingIntermediate2022,tu2023VisualQueryTuning,wang2025UnderstandingDeepRepresentation}. This observation is particularly important: if we rely only on the deepest feature of the public backbone, we risk creating the feature-space shortcut discussed in Section~\ref{sec:why-fail}. 

To limit this shortcut, we aggregate multiple levels of the frozen backbone. Earlier layers retain local patterns and structure that deeper layers gradually discard, while later layers supply semantic features. Their combination exposes broader input information than the final representation alone.

This richer representation introduces a second challenge: raw multi-level features are too large to use directly. Concatenating several intermediate activation maps can produce millions of values for a single input, which would substantially increase communication and may force the encrypted predictor to become large enough to cancel the efficiency gain of offloading. We address this problem by compressing each selected feature map before concatenation. For each intermediate representation $E_i(X)$ from the $i$-th selected layer of the public backbone $E$, we apply a random projection and concatenate the projected features:
\begin{equation}
\label{eq:feat-proj}
\tilde{X}=\text{Concat}(M_1E_1(X),M_2E_2(X),\dots,M_LE_L(X)),
\end{equation}
where $L$ is the number of included layers and each $M_i$ projects the corresponding feature into a $D$-dimensional vector.

The Johnson--Lindenstrauss result motivates random projection because a sufficiently large projected dimension preserves pairwise geometry with high probability~\cite{binghamRandomProjectionDimensionality2001}. Here $D$ is an empirical richness--cost parameter: increasing it generally improves preservation but increases communication and encrypted computation. We implement each fixed, public projection with the Fastfood transform~\cite{le2013FastfoodApproximatingKernel}, a structured approximation that avoids storing and multiplying by a dense random matrix. Therefore, the projection is a compression mechanism, not a learned defense.

\noindent\textbf{Estimating preserved information.} After projection, another question remains: how do we assess whether the representation is becoming richer? There is no task-independent threshold that establishes the absence of a shortcut. We therefore use MI estimation~\cite{songUnderstandingLimitationsVariational2020} as a practical assessment of information preservation (Definition~\ref{def:invariance}). Specifically, we define the normalized preserved mutual information (NPMI) as
\begin{equation}
\mathrm{NPMI}=\frac{I(X;\tilde{X})}{I(X;\tilde{X}_{\text{lossless}})},
\label{eq:npmi}
\end{equation}
where $\tilde{X}$ is the exposed representation and $\tilde{X}_{\text{lossless}}$ is a reference encoding with $I(X;\tilde{X}_{\text{lossless}})=H(X)$. NPMI is estimated offline and interpreted relatively: a larger value indicates that the estimator attributes more retained input information to the representation. Reported values are capped at 1.0 relative to the reference estimate. A low estimate may prompt the designer to include more feature levels or increase $D$, subject to system cost. Figure~\ref{fig:feature-mi} shows how the estimate changes as intermediate levels are included. MI estimation serves as an offline design diagnostic rather than an inference-time operation.

\input{data/mi-layer-num}

Table~\ref{tab:NPMI-med} reports the estimated NPMI values. The headline system results use the 104k/88k configurations and yield the reported 0.4--2.2-second latency range. The larger 468k Caltech101 and 396k Organ configurations serve as offline richness diagnostics. Together, these settings illustrate how increasing representation richness selects a different operating point in the richness--cost tradeoff.

\input{data/nmi-est}

\noindent\textit{Implementation.} We obtain the pretrained backbones from Timm~\cite{rw2019timm}. For ResNet50v2, we select the outputs of every activation together with the stem, normalization, global-pooling, and classifier layers, yielding $L=52$ feature groups. For MobileNetV3, we select all ReLU, Hardsigmoid, and Hardswish outputs together with the stem, global-pooling, flatten, and classifier layers, yielding $L=50$ and $44$ groups for the Large and Small variants. Each group is projected to $D=2{,}000$ values in the main configurations, so $D$ is the \emph{per-layer} projection size and the total dimension is $L\times D$: 104{,}000 for ResNet50v2, 100{,}000 for MobileNetV3-Large, and 88{,}000 for MobileNetV3-Small. We adjust the included levels or $D$ offline using the relative NPMI estimate together with latency and communication costs.

\subsection{Stage 2: Encrypted proprietary inference}\label{subsec:latent-model}

After the public encoder produces the latent representation $\tilde{X}$, the server no longer needs to evaluate a large end-to-end network under FHE. Instead, it only evaluates a compact proprietary predictor over $\tilde{X}$. We refer to this protected predictor as the \emph{latent model}, since it operates in the representation space produced by the public encoder and outputs the latent result $\tilde{Y}$ for later reconstruction.

Pretrained intermediate representations can support the latent model with a much smaller or even linear classifier~\cite{evciHead2ToeUtilizingIntermediate2022,tu2023VisualQueryTuning,wang2025UnderstandingDeepRepresentation}, so the protected component need not relearn feature extraction.

\noindent\textbf{Hierarchical feature gating.} This design still faces an important training challenge. The latent representation $\tilde{X}$ is formed by concatenating projected activations from different backbone levels. These features may have substantially different magnitudes and statistical properties, as illustrated in Figure~\ref{fig:layer-norms}. If they are directly fed into the latent model, features with larger norms may dominate optimization even when they are not the most useful for the downstream task, while informative but smaller-scale features may be underused. This imbalance can make training unstable and may also increase overfitting, especially when the target task depends unevenly on different abstraction levels.

\input{data/layer_norms}

To address this issue, we introduce a lightweight \emph{Hierarchical Feature Gating} (HFG) layer at the input of the latent model. Let $E'_i(X)=M_iE_i(X)$ denote the projected feature from the $i$-th selected layer, as defined in Equation~\ref{eq:feat-proj}. HFG rescales the contribution of each selected layer before concatenation:
\begin{equation}
\text{HFG}(\tilde{X})=\text{Concat}(w_1E'_1(X),\dots,w_LE'_L(X)),
\end{equation}
where $w_1,\dots,w_L$ are learned nonnegative gating weights inside the proprietary predictor. In the released implementation, each gate is obtained independently from a trainable parameter $q_i$:
\begin{equation}
w_i=\operatorname{sigmoid}(q_i).
\end{equation}
This mechanism gives the latent model an explicit way to balance features from different abstraction levels. Rather than forcing the following fully connected layer to compensate implicitly for scale differences, HFG rescales each feature group. The predictor can therefore emphasize useful levels and suppress noisy ones. Section~\ref{subsec:hfg-ablation} evaluates its effect on convergence and overfitting.

HFG improves training without increasing the cryptographic cost of inference. Once training is finished, the learned gating operation is only a fixed linear rescaling of the input feature groups. Therefore, it can be absorbed into the first fully connected layer of the latent model:
\begin{equation}
W_{\text{new}}=W_{\text{old}}W_{\text{HFG}},
\end{equation}
where $W_{\text{HFG}}$ is the matrix form of the learned gating transformation. After this absorption, encrypted inference does not need an additional FHE operation for HFG.  

\noindent\textit{FHE instantiation.} The protected computation uses pure FHE and a single interaction round. Following prior work~\cite{brutzkusLowLatencyPrivacy2019}, we densely pack the representation into ciphertext slots. The headline one-layer predictor is linear after HFG fusion; for supported multilayer predictors, nonlinear layers use a degree-2 polynomial approximation of ReLU~\cite{hesamifardCryptoDLDeepNeural2017}. The lightweight fully connected computation is parallelized across output channels.

\subsection{Stage 3: Output mapping}\label{subsec:output-mapping}

The intended API output is the class-probability vector, which the user is allowed to recover. Returning raw logits would reveal one unnecessary degree of freedom: for any constant $c$, $\operatorname{softmax}(z)=\operatorname{softmax}(z+c)$. The output interface therefore removes this common offset and reveals exactly the information needed to reconstruct the intended probabilities.

\noindent\textbf{Bijective logit mapping.} To avoid this, we re-center the latent logits before returning them. Let $\text{logits}$ denote the output of the latent model over $C$ classes. We compute
\begin{equation}
\tilde{Y}_i=\text{logits}_i-\frac{1}{C}\sum_{j=1}^{C}\text{logits}_j.
\end{equation}
This fixed, parameter-free operation removes the softmax-invariant offset and leaves the relative logit differences. The centered logits are in one-to-one correspondence with the class-probability vector, i.e., $Y=\text{softmax}(\tilde{Y})
\Leftrightarrow
\tilde{Y}=\log Y-\frac{1}{C}\sum_{j=1}^{C}\log Y_j$. Thus the user intentionally recovers $Y$, but receives no additional absolute-logit offset. The same mapping is used for every query and does not change prediction correctness.

Together, these stages move a \textit{task-agnostic}, feature-rich public encoder outside encrypted computation while retaining only the compact proprietary predictor under FHE.

%% file: data/mi-layer-num.tex
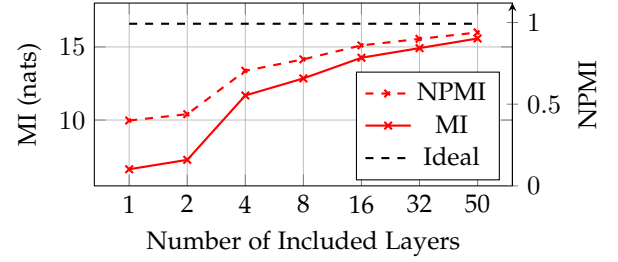
\begin{figure}[bt]
\centering
\begin{tikzpicture}
  \begin{axis}[
      name=MIaxis,
      width=0.8\columnwidth,
      height=4.0cm,
      xlabel={Number of Included Layers},
      ylabel={MI (nats)},
      ymax=18,
      symbolic x coords={1,2,4,8,16,32,50},
      xtick=data,
      legend pos=south east,
      grid=both,
      grid style={line width=.1pt, draw=gray!10},
      major grid style={line width=.2pt, draw=gray!50},
  ]
    \addlegendimage{mark=x,red,dashed,thick}
    \addlegendentry{NPMI}
    \addplot[
      color=red,
      mark=x,
      thick,
    ]
    coordinates {
      (1,  6.6349961222434530)
      (2,  7.2765701303676690)
      (4,  11.680898974541904)
      (8,  12.847130804645772)
      (16, 14.261099607766079)
      (32, 14.918647136817986)
      (50, 15.576738454857649)
    };
    \addlegendentry{MI}
    
   \addplot[
      dashed,
      black,
      thick,
    ]
    coordinates {(1,16.5851) (50,16.5851)};
    \addlegendentry{Ideal}

  \end{axis}
  \begin{axis}[
      at={(MIaxis.south west)},
      anchor=south west,
      width=0.8\columnwidth,
      height=4.0cm,
      xlabel={},
      ylabel={NPMI},
      y label style={at={(axis description cs:1.05,.5)},anchor=south},
      ymin=0, ymax=1.119,
      symbolic x coords={1,2,4,8,16,32,50},
      xtick=data,
      axis x line=none,
      axis y line=right,
      legend style={at={(0.5,-0.35)},anchor=north,legend columns=1}
  ]
    \addplot[
      dashed,
      mark=x,
      red,
      thick,
    ]
    coordinates {
      (1, 0.4000576494711188)
      (2, 0.4387414082741538)
      (4, 0.7043007865217517)
      (8, 0.7746188328466981)
      (16, 0.8598742008047029)
      (32, 0.8995210843961137)
      (50, 0.9392007557902966)
    };
  \end{axis}
\end{tikzpicture}
\caption{The preserved information for CIFAR10 when including features from different numbers of MobileNetV3 layers.}
\label{fig:feature-mi}
\end{figure}

%% file: data/nmi-est.tex
\begin{table}[htb]
\centering
\caption{Estimated NPMI across latent dimensions.}
\label{tab:NPMI-med}

\resizebox{\linewidth}{!}{
\begin{tabular}{ccccccccc}
\toprule
Task & Pets & MNIST & EuroSAT & CIFAR100 & FSDD & \multicolumn{3}{c}{Caltech101} \\
\midrule
Dimension & 104k & 104k & 104k & 104k & 104k & 104k & 208k & 468k \\
$\mathrm{NPMI}$ & 0.95 & 1.0 & 1.0 & 0.99 & 0.96 & 0.73 & 0.90 & 0.97 \\
\end{tabular}
}
\resizebox{\linewidth}{!}{
\begin{tabular}{ccccccccc}
\toprule
Task & Blood & Path. & Breast & Skin & PN & \multicolumn{3}{c}{Organ} \\
\midrule
Dimension & 88k & 88k & 88k & 88k & 88k & 88k & 352k & 396k \\
$\mathrm{NPMI}$ & 1.0 & 0.93 & 0.95 & 1.0 & 1.0 & 0.88 & 0.93 & 0.99 \\
\bottomrule
\end{tabular}
}
\end{table}

%% file: data/layer_norms.tex
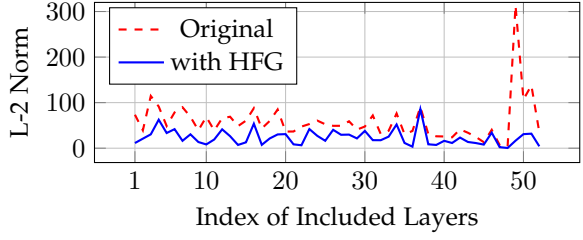
\begin{figure}[bt]
\centering
\begin{tikzpicture}
  \begin{axis}[
      width=0.9\columnwidth,
      height=3.7cm,
      xlabel={Index of Included Layers},
      ylabel={L-2 Norm},
      ymax=320,
      xtick={1, 10, 20, 30, 40, 50},
      legend pos=north west,
      grid=both,
      grid style={line width=.1pt, draw=gray!10},
      major grid style={line width=.2pt, draw=gray!50},
  ]
  
    \addplot[
      color=red,
      dashed,
      thick,
    ]
    coordinates {
        (1,73.46071)
        (2,39.203716)
        (3,114.656845)
        (4,91.42586)
        (5,48.808697)
        (6,78.51054)
        (7,89.88577)
        (8,68.922226)
        (9,39.322006)
        (10,68.684456)
        (11,38.926563)
        (12,65.80317)
        (13,69.06015)
        (14,48.483955)
        (15,60.416687)
        (16,87.75376)
        (17,43.34796)
        (18,59.851665)
        (19,85.01636)
        (20,36.480286)
        (21,36.62792)
        (22,47.420322)
        (23,52.99015)
        (24,60.53)
        (25,52.52801)
        (26,49.07523)
        (27,49.13526)
        (28,59.08541)
        (29,40.7994)
        (30,47.5363)
        (31,71.88845)
        (32,31.124262)
        (33,38.27204)
        (34,76.21674)
        (35,30.898918)
        (36,37.703175)
        (37,90.67243)
        (38,30.90218)
        (39,26.123487)
        (40,25.699663)
        (41,23.896372)
        (42,41.338905)
        (43,33.874435)
        (44,24.702753)
        (45,13.176276)
        (46,40.29451)
        (47,6.9098835)
        (48,3.285964)
        (49,313.31427)
        (50,106.77102)
        (51,137.9023)
        (52,34.914757)
    };
    \addlegendentry{Original}
    
    \addplot[
      color=blue,
      thick,
    ]
    coordinates {
        (1,11.320293)
        (2,20.907827)
        (3,30.396406)
        (4,62.526787)
        (5,33.30708)
        (6,41.79992)
        (7,16.201462)
        (8,30.416967)
        (9,13.50847)
        (10,8.256589)
        (11,19.084265)
        (12,41.389366)
        (13,26.379505)
        (14,7.128356)
        (15,12.849989)
        (16,53.740692)
        (17,7.3695583)
        (18,21.383251)
        (19,30.195293)
        (20,30.883871)
        (21,8.442646)
        (22,6.4531856)
        (23,41.985733)
        (24,27.094893)
        (25,16.252378)
        (26,40.331104)
        (27,29.424002)
        (28,29.829824)
        (29,21.467281)
        (30,37.956463)
        (31,17.625546)
        (32,17.499035)
        (33,25.394014)
        (34,52.024384)
        (35,11.463617)
        (36,3.3670547)
        (37,85.61391)
        (38,8.518342)
        (39,7.127787)
        (40,16.160234)
        (41,11.4106455)
        (42,23.30609)
        (43,13.624376)
        (44,11.279297)
        (45,7.8937483)
        (46,34.075848)
        (47,2.2641864)
        (48,0.4389036)
        (49,16.452078)
        (50,30.809614)
        (51,31.884787)
        (52,4.141429)
    };
    \addlegendentry{with HFG}
  \end{axis}
\end{tikzpicture}
\caption{Representation norms across layers before and after applying the HFG layer.}
\label{fig:layer-norms}
\end{figure}

%% file: sections/exps.tex
\section{Experiments}\label{sec:experiments}

We evaluate \methodname's efficiency and utility across realistic tasks, then separately assess its incremental extraction risk under concrete component-reuse attacks. Our implementation is available at \url{https://github.com/CassiniHuy/gecko-infer}.

\subsection{Experiment Setup}\label{subsec:exp-setup}

\noindent\textbf{FHE Scheme and Parameter Setting.} We use the FHE library Lattigo~\cite{mouchetLattigoMultipartyHomomorphic2020} with the CKKS scheme~\cite{ckks17asiacrypt}, the conjugate-invariant variant~\cite{kimApproximateHomomorphicEncryption2018} with a polynomial degree of $N=8192$, a total ciphertext modulus of approximately $2^{218}$, and a plaintext scale of $2^{30}$. These parameters ensure at least 128-bit security~\cite{albrechtHomomorphicEncryptionStandard2021}. The baseline approaches~\cite{kimOptimizedPrivacyPreservingCNN2023,usenix22cheetah,he2024RhombusFastHomomorphic} use their original parameter settings under the same security level.

\noindent\textbf{Hardware Environment.} All experiments are conducted on a Linux machine with dual 2.10\,GHz CPUs and 128\,GB of memory. We use a CPU-only environment for inference-time evaluation to ensure fair overhead comparisons, although GPUs could further accelerate the public encoder. Reported latency numbers are averaged over 16 test runs.

\noindent\textbf{Neural Network Architectures.} ResNet50v2~\cite{he2016identity} serves as the default public backbone in \methodname. We also evaluate MobileNetV3~\cite{howard2019searching} as a compact alternative designed for resource-constrained devices. This combination allows us to study both a strong, widely used pretrained backbone and a lightweight practical one. All pretrained models are obtained from the Timm library~\cite{rw2019timm}.

\noindent\textbf{Tasks and Datasets.} To evaluate \methodname in realistic settings, we use a diverse set of classification tasks spanning natural images, biomedical images, and audio spectrograms~\cite{krizhevsky2009learning,deng2012mnist,parkhi2012cats,lecun1998gradient,helber2019eurosat,zohar_jackson_2018_1342401,medmnistv2}. These tasks cover input resolutions up to $224\times224$ and up to 101 classes, and include privacy-sensitive biomedical applications. For each dataset, we follow the original train/test split when available. We fine-tune for 100 epochs with learning rates ranging from 0.0001 to 0.01 and report the final test accuracy.

\noindent\textbf{Metrics.} We evaluate the following dimensions:
\begin{itemize}[leftmargin=*]
    \item \textit{Accuracy}: accuracy of the private inference service.
    \item \textit{Accuracy (Baseline)}: accuracy of a standard transfer-learning baseline that trains a classifier on the final layer of the same public backbone under the same settings.
    \item \textit{Crypto. MACs}: number of multiply-accumulate operations (MACs) performed under cryptographic protection.
    \item \textit{Reduced MACs}: number of protected MACs reduced by \methodname relative to the corresponding baseline for the same task, backbone, and input size.
    \item \textit{Communication}: total data exchanged between the client and server per inference, without pre-computation.
    \item \textit{End-to-End Time}: total latency per private inference, from the raw input to the final output, without precomputation.
\end{itemize}

Our evaluation is designed to reflect realistic use cases rather than only small benchmark settings. Besides CIFAR-like standard datasets, we include higher-resolution image tasks such as Caltech101, as well as biomedical image applications where private inference is especially meaningful. 

\subsection{Performance Comparison to Baselines}\label{subsec:compare-to-baseline}

We begin with CIFAR10 because it is the most widely used benchmark in prior private-inference studies, which enables the most direct comparison with existing baselines. We then extend the evaluation to more realistic real-world tasks in the following subsections.

\noindent\textbf{Baseline Setup.} This comparison measures how much encrypted computation can be removed while protecting proprietary functionality. The end-to-end baselines protect a complete ResNet~\cite{he2016deep}, including its generic feature extractor. In contrast, \methodname moves only an independently available, frozen backbone outside FHE; all task-specific parameters and the proprietary input--output mapping remain in the protected predictor. Thus, although \methodname cryptographically protects fewer operations, it protects the same type of proprietary functionality, maintains comparable task accuracy, and provides comparable extraction robustness under the evaluated attacks in Section~\ref{subsec:sec-eval}. For smaller ResNets, we use an optimized pure-FHE baseline~\cite{kimOptimizedPrivacyPreservingCNN2023}, which supports models up to ResNet20. For ResNet50 with $224\times224$ inputs, we use the hybrid Cheetah~\cite{usenix22cheetah} and Rhombus~\cite{he2024RhombusFastHomomorphic} results in their reported LAN setting. \methodname uses comparable ResNet18 and ResNet50v2 public backbones.

Naive final-layer offloading can achieve similar or lower overhead, but it exposes the selective representation consumed by its private head. At matched representation dimensions, our extraction experiments find this shortcut substantially easier to imitate than \methodname for query budgets of 200--5{,}000 (Figure~\ref{fig:attack-deep-rep}). It therefore trades model-side robustness for efficiency and is evaluated in the extraction ablation rather than included as a performance baseline here.

\input{data/comparison}

\noindent\textbf{Results.} Table~\ref{tab:comparison} shows that moving independently public computation outside FHE reduces cryptographic cost while maintaining comparable task accuracy. Crypto. MACs fall from 41.32M/4.13G in the end-to-end baselines to 38.01K/1.04M in the ResNet18/ResNet50v2 \methodname configurations. Corresponding end-to-end latency is 0.33\,s and 0.99\,s, versus hundreds of seconds for the complete protected networks. This advantage is complementary to lower-level cryptographic optimizations~\cite{he2024RhombusFastHomomorphic,juNeuJeansPrivateNeural2024}.

\subsection{Generalization Across Tasks and Backbones}\label{subsec:eval-diverse-tasks}

\noindent\textbf{Image tasks with different backbones.} To show that \methodname is not tied to a single benchmark or backbone, we evaluate it on a diverse set of image classification tasks with input resolutions up to $224\times224$ and up to 101 classes, including Pets~\cite{parkhi2012cats}, MNIST~\cite{lecun1998gradient}, EuroSAT~\cite{helber2019eurosat}, CIFAR100~\cite{krizhevsky2009learning}, and Caltech101~\cite{fei2004learning}. We instantiate \methodname with two widely used public backbones: ResNet50v2 and MobileNetV3~\cite{howard2019searching}. MobileNetV3 is included to evaluate whether the public encoder can remain practical even in resource-constrained settings, since it contains only 5.48M parameters, substantially fewer than ResNet50v2's 25.56M.

\input{data/diverse-tasks}

\noindent\textbf{Results.} Table~\ref{tab:diverse-tasks} shows that \methodname consistently maintains accuracy close to the corresponding transfer-learning baselines across all evaluated tasks, while greatly reducing the amount of protected computation. Across these image tasks, end-to-end private inference completes in 0.78--2.20 seconds with communication of at most 10.8\,MB. The results also show that \methodname is robust to the backbone choice: while ResNet50v2 delivers the strongest default performance, the compact MobileNetV3 backbone remains highly competitive, indicating that \methodname does not rely on a single large public encoder to be effective.

\noindent\textbf{From image to audio.} To further test whether \methodname is limited to vision backbones and tasks, we evaluate it on the FSDD audio task~\cite{zohar_jackson_2018_1342401}. We consider two public encoders: an ImageNet-pretrained ResNet50v2 and Google's AudioSet-pretrained YAMNet~\cite{yamnet,gemmeke2017AudioSetOntology}, which uses a MobileNetV1 backbone with only 3.73M parameters. \methodname achieves 96.33\% accuracy with 0.62\,s latency using YAMNet, and still reaches 96.76\% accuracy with 0.97\,s latency even when the public encoder is an image-pretrained ResNet50v2. These results suggest that \methodname can work with readily available off-the-shelf encoders beyond a single modality.

\noindent\textbf{Biomedical case study.} We next evaluate \methodname on six biomedical image processing tasks~\cite{medmnistv2} covering different biomedical image types and modalities at resolution $128\times128$. This setting is particularly relevant because the inputs are privacy-sensitive and realistic private inference is therefore especially valuable. To keep the public component lightweight, we use MobileNetV3-Small~\cite{rw2019timm} as the public backbone, which has only 1.59M parameters, and use a single fully connected layer as the latent model.

\input{data/medical}

Table~\ref{tab:mobilenet-medical} shows that \methodname remains accurate and efficient in these biomedical settings. Across all six tasks, its accuracy stays within 1\% of the corresponding baseline, while the encrypted computation remains below 1M Crypto.\ MACs and the total latency is at most 604\,ms. These results show that \methodname generalizes across various realistic tasks and privacy-sensitive biomedical applications, while remaining effective with both strong and compact public backbones.

\subsection{Security Evaluation: Extraction Risk Analysis}
\label{subsec:sec-eval}

We now ask whether \methodname's public components improve extraction over ordinary black-box access. This is an incremental-risk question, not a claim that \methodname eliminates model extraction.

\noindent\textbf{Attack setup.} On Caltech101, the model-extraction adversary queries valid ImageNet inputs~\cite{russakovsky2015ImageNetLargeScale} and distills the returned class-probability vectors, following Knockoff Nets~\cite{orekondy2019knockoff}. Every surrogate uses the predictor architecture of its corresponding target; \methodname and its component-reuse surrogates use the same one-layer private head. Surrogates train for 30 epochs under the same query budgets shown in Figure~\ref{fig:attack-compare}. We report surrogate accuracy on the true task labels because it measures attacker utility.

\textit{(a) Comparison to the baseline model.} This experiment asks whether \methodname itself becomes an easier extraction target than a standard baseline. We compare two target models: a baseline model trained by standard transfer learning with a ResNet50v2 backbone, whose task accuracy is 94.54\%, and a \methodname target using the same ResNet50v2 backbone as the public encoder, whose task accuracy is 95.48\%. For both targets, the attacker uses the same pretrained backbone architecture as the target model, initializes it with ImageNet-pretrained weights, and trains the whole surrogate model. This setting evaluates whether replacing the baseline model with \methodname increases the attacker's extraction success under the same public-resource assumption.

\textit{(b) Public-component reuse attacks.} This experiment fixes \methodname as the target and progressively reuses its public components. The \textit{free-tuning attack} initializes the same pretrained backbone architecture and fine-tunes the complete surrogate. The \textit{backbone-reuse attack} instead reuses and freezes \methodname's exposed backbone, training only the downstream predictor. The \textit{encoder-reuse attack} reuses and freezes the complete public encoder, including the exact projection. These variants test whether direct reuse helps under the evaluated optimization strategies.

\input{data/attack-comparison}

\noindent\textbf{Results.} Figure~\ref{fig:attack-compare-baseline} shows lower surrogate task accuracy against \methodname than against the transfer-learning target at every tested budget. At 6400 queries, the values are 85.71\% and 93.11\%, respectively, while \methodname itself has slightly higher target accuracy. The 85.71\% result is substantial ordinary extraction risk, not a security threshold; the relevant observation is that replacing the baseline target with \methodname did not increase attacker utility in this experiment.

Figure~\ref{fig:attack-reuse} shows 85.71\%, 83.38\%, and 55.81\% for free tuning, backbone reuse, and encoder reuse at 6{,}400 queries. Free tuning achieves the highest accuracy at every budget from 200 onward because it can optimize the complete surrogate, whereas the reuse variants freeze the exposed components. Across the tested budgets, component reuse therefore offers no consistent extraction advantage.

\subsection{Security Ablations for Public Offloading}
\label{subsec:sec-ablation}

We next examine two risks highlighted by our formalization: exposing prediction-related artifacts through the public component, and creating a feature-space shortcut through an overly selective representation. We study both risks under practical model-extraction settings with query access to the target model and publicly available substitute data.

\noindent\textbf{Prediction-artifact independence ablation.}
We first examine whether exposing additional prediction-related artifacts can help an attacker through two experiments:

\begin{itemize}[leftmargin=*]
    \item \textit{Exp. 1 (Increasing Component Reuse)}: We test whether the public encoder used by \methodname leaks exploitable information. The attacker builds the surrogate model under three settings with increasing levels of reuse: (1) \textit{same \methodname design}, where the attacker follows the same public-encoder design but uses a different backbone and projection; (2) \textit{+same backbone}, where the attacker further uses the same backbone as the target but reinitializes the projection parameters; and (3) \textit{+same projection}, where the attacker also uses the same projection as the \methodname target, i.e., exactly the same public encoder exposed. If the exposed encoder contains exploitable information, increasing the overlap with the target encoder should improve extraction performance; otherwise, the results should remain similar across these settings.
    \item \textit{Exp. 2 (Partial-weight exposure)}: we construct an insecure baseline that exposes incomplete but informative model weights. Following prior work~\cite{renPrivDNNSecureMultiParty2024}, we protect only a subset of neurons in the first two layers and expose the remaining model. The protected neurons are selected by importance under structured pruning~\cite{anwar2017structured}. Note that the exposed partial model is no longer independently usable for accurate inference.
\end{itemize}

Figure~\ref{fig:security-ablations} reports the results on CIFAR10 using STL10~\cite{coates2011analysis} as the substitute query set. Figure~\ref{fig:attack-Independence-gecko} shows no consistent accuracy increase as the attacker reuses more of the public encoder, indicating that it provides little exploitable deployment-specific information in this setting. In contrast, Figure~\ref{fig:attack-Independence-leakage} shows that exposing informative model weights substantially helps extraction: at 5{,}000 queries, partial exposure raises attack accuracy by approximately 43--44 percentage points over full protection.

\noindent\textbf{Feature richness ablation.} We next examine whether a lack of feature richness leads to the feature-space shortcut problem, which makes the proprietary model easier to imitate.

\begin{itemize}[leftmargin=*]
    \item \textit{Exp. 3 (Naive representation)}: we replace \methodname's multi-level latent representation with a naive deep representation built only from the exposed backbone's final features, following prior encoder-splitting designs~\cite{brutzkusLowLatencyPrivacy2019,ruanPrivateEfficientAccurate2023}. The representation dimension is kept at the same level as that of \methodname, so the main difference is the amount of preserved input information rather than the raw feature size.
    \item \textit{Exp. 4 (Varying feature richness)}: we study whether richer public representations make the remaining private predictor harder to extract. We vary the latent dimension to control how much input information is preserved in the \methodname public representation, and use the estimated input--representation MI as a proxy for feature richness. To avoid underfitting due to the increased feature dimension, we ensure every surrogate model has converged.
\end{itemize}

Figure~\ref{fig:attack-deep-rep} is our most controlled shortcut comparison: at similar total representation dimensions (about 100k), the naive final-layer target yields 19--27 percentage points higher surrogate task accuracy for query budgets of 200--5{,}000. This is consistent with the selective-interface risk. In Figure~\ref{fig:attack-invariance-mi}, the largest tested representation has higher estimated MI and 26--29 percentage points lower surrogate accuracy than the smallest representation.

\input{data/attack-sec-ablation}

\subsection{Representation Design Ablation}\label{subsec:repre-ablation}

Figure~\ref{fig:metrics-dim} shows the effect of varying the \emph{total} latent dimension from 2{,}000 to 104{,}000 with the public backbone fixed. Accuracy is nearly stable from about 20{,}000 onward, whereas estimated MI rises from 4.01 to 11.78 nats. Latency and communication concurrently rise from about 0.3\,s to 1.0\,s and from 1.3\,MB to 10.8\,MB. This is distinct from the main configuration's $D=2{,}000$ \emph{per-layer} projection size, which produces 104{,}000 total values for ResNet50v2.

\input{data/metrics-dim}

\subsection{Client-side Overhead Analysis}\label{subsec:client-overhead-analysis}

Figure~\ref{fig:client-cost-dim} reports the client-side overhead introduced by \methodname. Since the output mapping is negligible, the main practical cost is executing the public encoder. Under our CPU-only setting, client-side processing remains below 0.5\,s, with memory usage ranging from roughly 940\,MiB to 1{,}368\,MiB across the evaluated latent dimensions. Thus, the client overhead remains practical while increasing in a predictable way with richer exposed representations. This overhead can be further reduced through GPU acceleration, lightweight public backbones such as MobileNet variants, and standard compression techniques such as quantization~\cite{wu2016quantized}. It can also be offloaded to the server by protocols that protect only user privacy for public encoding.

\input{data/client-cost}

\subsection{HFG Ablation}\label{subsec:hfg-ablation}

Figure~\ref{fig:eval-hfg} compares latent-model training with and without HFG while holding all other settings fixed. On CIFAR10, HFG produces a smoother test-accuracy trajectory and raises the final test accuracy from 90.10\% to 92.67\%. The smaller Pets dataset exposes the overfitting effect more clearly: without HFG, training accuracy reaches 100\% while test accuracy remains around 70\%; with HFG, test accuracy exceeds 90\%. These results show that balancing feature groups improves optimization stability and generalization.

\input{data/hfg}

%% file: data/comparison.tex
\begin{table}[htb]
\centering
\caption{CIFAR10 efficiency comparison. \methodname moves only an independently public generic encoder outside FHE while protecting all task-specific proprietary functionality and maintaining comparable accuracy. Reduced MACs are measured against a complete network matching each \methodname backbone; the displayed columns report representative prior private-inference systems. Naive final-layer offloading offers similar efficiency but is substantially easier to extract, so we evaluate it separately in Section~\ref{subsec:sec-ablation}.}
\label{tab:comparison}
\resizebox{\linewidth}{!}{
\begin{tabular}{lccccc}
\toprule
 & \multicolumn{2}{c}{\methodname} & \multicolumn{3}{c}{End-to-End Baseline}\\
\midrule
Crypto. Scheme & \multicolumn{2}{c}{Pure-FHE (CKKS)} & Pure-FHE~\cite{kimOptimizedPrivacyPreservingCNN2023}  &  Hybrid~\cite{usenix22cheetah} & Hybrid~\cite{he2024RhombusFastHomomorphic} \\
Architecture & ResNet18  &  ResNet50v2 &  ResNet20  & \multicolumn{2}{c}{ResNet50}\\
\midrule
Accuracy    &   89.48\%    & 93.51\%    & 90.39\%    & 93.62\%  & 93.62\%     \\
Crypto. MACs &   38.01K   & 1.04M     & 41.32M     & 4.13G    & 4.13G      \\
Reduced MACs &   37.22M   & 4.09G     & N/A     & N/A   & N/A      \\
Communication &  4.45\,MB    & 10.8\,MB & 3.15\,MB   & 1{,}815\,MB   & 1{,}152\,MB     \\
End-to-End Time &  \textbf{0.33\,s} & \textbf{0.99\,s} & 644\,s     & 188\,s  & 145\,s      \\
\bottomrule
\end{tabular}
}
\end{table}

%% file: data/diverse-tasks.tex
\begin{table}[htb]
\centering
\caption{Performance of \methodname on various tasks.}
\label{tab:diverse-tasks}

\resizebox{\linewidth}{!}{
\begin{tabular}{lccccc}
\toprule
Task& MNIST & EuroSAT & CIFAR100 & Pets  & Caltech101 \\
\midrule
Resolution & 28 & 64 & 32 & 224 & 224 \\
Categories & 10 & 10 & 100 & 37 & 101 \\
\midrule
\midrule
 \multicolumn{6}{c}{ResNet50v2 as Backbone} \\
\midrule
Accuracy             & 96.55\%  & 95.03\%  & 78.41\%   & 90.19\%   & 95.48\%    \\
Accuracy (Baseline)  & 96.53\%  & 94.85\%  & 76.67\%   & 90.92\%   & 94.54\%    \\
\midrule
Crypto. MACs         & 1.04M    & 1.04M    & 10.00M     & 3.85M    & 10.50M     \\
Reduced MACs          & 79.16M    & 333.65M    & 83.42M     & 4.09G   & 4.09G     \\
Communication        & 10.8\,MB & 10.8\,MB & 10.8\,MB  & 10.8\,MB  & 10.8\,MB \\
End-to-End Time      & 0.98\,s  & 0.97\,s  & 2.17\,s    & 1.28\,s  & 2.20\,s  \\
\midrule
\midrule
 \multicolumn{6}{c}{MobileNetV3 as Backbone} \\
\midrule
Accuracy             & 97.27\%  & 94.41\%    & 76.08\%     & 89.83\%  & 95.02\%    \\
Accuracy (Baseline) & 97.85\%  & 94.30\%  & 75.61\%   & 88.23\%   & 95.04\%    \\
\midrule
Crypto. MACs           & 1.00M    & 1.00M      & 10.00M    & 3.70M    & 10.10M     \\
Reduced MACs        & 6.85M    & 20.09M    & 7.08M       & 215.36M   & 215.36M     \\
Communication        & 10.8\,MB & 10.8\,MB & 10.8\,MB  & 10.8\,MB  & 10.8\,MB \\
End-to-End Time       & 0.78\,s  & 0.78\,s    & 2.00\,s   & 1.09\,s   & 2.01\,s  \\
\bottomrule
\end{tabular}
}
\end{table}

%% file: data/medical.tex
\begin{table}[htb]
\centering
\caption{Performance of \methodname on biomedical tasks.}
\label{tab:mobilenet-medical}

\resizebox{\linewidth}{!}{
\begin{tabular}{lcccccc}
\toprule
Task                 & Blood    & Organ    & Path.    & Breast     & Skin     & PN\\
\midrule
Modality             & Micro.   & CT       & Hist.    & US         & Derm.    & X-Ray\\
Resolution           & 128        & 128       & 128        & 128          & 128        & 128\\
Categories           & 8        & 11       & 9        & 2          & 7        & 2\\
\midrule
Accuracy             & 95.73\%  & 76.28\%  & 88.48\%  & 87.18\%    & 78.75\%  & 87.66\% \\
Accuracy (Baseline)             & 94.94\%  & 75.94\%  & 88.30\%  & 87.82\%    & 78.75\%    & 88.62\%    \\
\midrule
Crypto. MACs         & 704K     & 968K     & 792K     & 176K       & 616K     & 176K \\
Reduced MACs         & 7.85M     & 7.85M     & 7.85M     & 7.85M       & 7.85M     & 7.85M \\
Communication        & 9.2\,MB  & 9.2\,MB  & 9.2\,MB  & 9.2\,MB    & 9.2\,MB  & 9.2\,MB \\
End-to-End Time      & 575\,ms  & 604\,ms  & 507\,ms  & 413\,ms    & 555\,ms  & 429\,ms \\
\bottomrule
\end{tabular}
}
\end{table}

%% file: data/attack-comparison.tex
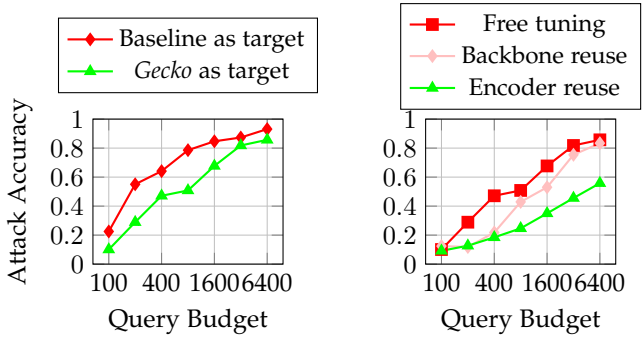
\begin{figure}[tb]
\centering
\subfloat[Comparison to baseline model\label{fig:attack-compare-baseline}]{%
\begin{minipage}[b]{0.48\columnwidth}
\centering
\begin{tikzpicture}
\begin{axis}[
    height=3.5cm,
    width=0.96\linewidth,
    xlabel={Query Budget},
    ylabel={Attack Accuracy},
    symbolic x coords={100,200,400,800,1600,3200,6400},
    ymin=0, ymax=1.0,
    legend style={at={(0.5,1.7)},anchor=north,legend columns=1,font=\small},
    grid=both,
    grid style={line width=.1pt, draw=gray!10},
    major grid style={line width=.2pt, draw=gray!50},
]

\addplot[mark=diamond*,red,thick] coordinates {
(100,0.2250)
(200,0.5505)
(400,0.6413)
(800,0.7858)
(1600,0.8469)
(3200,0.8730)
(6400,0.9311)
};
\addlegendentry{Baseline as target}

\addplot[mark=triangle*,green,thick] coordinates {
(100,0.1007)
(200,0.2891)
(400,0.4712)
(800,0.5081)
(1600,0.6770)
(3200,0.8181)
(6400,0.8571)
};
\addlegendentry{\methodname as target}

\end{axis}
\end{tikzpicture}
\end{minipage}}%
\hfill
\subfloat[Public-component reuse\label{fig:attack-reuse}]{%
\begin{minipage}[b]{0.48\columnwidth}
\centering
\begin{tikzpicture}
\begin{axis}[
    height=3.5cm,
    width=0.96\linewidth,
    xlabel={Query Budget},
    symbolic x coords={100,200,400,800,1600,3200,6400},
    ymin=0, ymax=1.0,
    legend style={at={(0.5,1.8)},anchor=north,legend columns=1,font=\small},
    grid=both,
    grid style={line width=.1pt, draw=gray!10},
    major grid style={line width=.2pt, draw=gray!50},
]

\addplot[mark=square*,red,thick] coordinates {
(100,0.1007)
(200,0.2891)
(400,0.4712)
(800,0.5081)
(1600,0.6770)
(3200,0.8181)
(6400,0.8571)
};
\addlegendentry{Free tuning}

\addplot[mark=diamond*,pink,thick] coordinates {
(100,0.1243)
(200,0.1180)
(400,0.2181)
(800,0.4290)
(1600,0.5300)
(3200,0.7540)
(6400,0.8338)
};
\addlegendentry{Backbone reuse}

\addplot[mark=triangle*,green,thick] coordinates {
(100,0.0917)
(200,0.1270)
(400,0.1846)
(800,0.2467)
(1600,0.3499)
(3200,0.4555)
(6400,0.5581)
};
\addlegendentry{Encoder reuse}

\end{axis}
\end{tikzpicture}
\end{minipage}}%
\caption{Surrogate task accuracy under the evaluated extraction settings. Target accuracies are 94.54\% (baseline) and 95.48\% (\methodname). (a) The same transfer-learning attack has lower accuracy against \methodname. (b) Backbone reuse freezes the public backbone; encoder reuse additionally freezes the exact projection. Reuse provides no consistent advantage over free end-to-end surrogate tuning and is less effective from 200 queries onward.}
\label{fig:attack-compare}
\end{figure}

%% file: data/attack-sec-ablation.tex
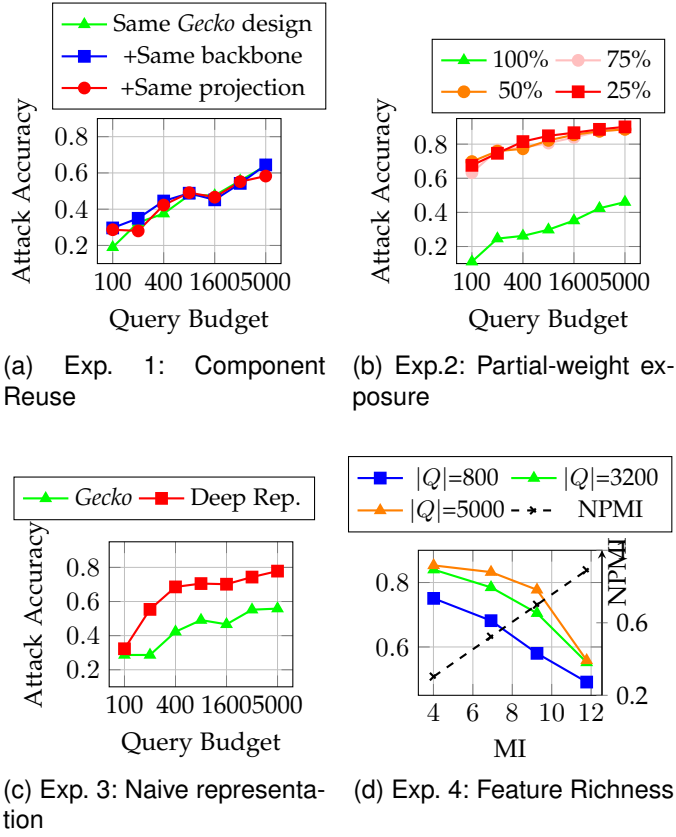
\begin{figure}[tb]
\centering
\subfloat[Exp. 1: Component Reuse\label{fig:attack-Independence-gecko}]{%
\begin{minipage}[b]{0.48\columnwidth}
\centering
\begin{tikzpicture}
\begin{axis}[
    width=0.94\linewidth,
    height=3.5cm,
    xlabel={Query Budget},
    ylabel={Attack Accuracy},
    symbolic x coords={100,200,400,800,1600,3200,5000},
    ymin=0.1, ymax=0.9,
    legend style={at={(0.5,1.8)},anchor=north,legend columns=1,font=\small},
    grid=both,
    grid style={line width=.1pt, draw=gray!10},
    major grid style={line width=.2pt, draw=gray!50},
]
\addplot[mark=triangle*,green,thick] coordinates {
    (100, 0.18979999423027039)
    (200, 0.32649999856948850)
    (400, 0.37540000677108765)
    (800, 0.47749999165534973)
    (1600,0.47569999098777770)
    (3200,0.55699998140335080)
    (5000,0.63980001211166380)
};
\addlegendentry{Same \methodname design}

\addplot[mark=square*,blue,thick] coordinates {
    (100, 0.29660001397132874)
    (200, 0.34970000386238100)
    (400, 0.44510000944137573)
    (800, 0.48809999227523804)
    (1600,0.45159998536109924)
    (3200,0.54240000247955320)
    (5000,0.64420002698898320)
};
\addlegendentry{+Same backbone}

\addplot[mark=*,red,thick] coordinates {
    (100, 0.2867)
    (200, 0.2802)
    (400, 0.4228)
    (800, 0.4911)
    (1600,0.4663)
    (3200,0.5520)
    (5000,0.5827)
};
\addlegendentry{+Same projection}
\end{axis}
\end{tikzpicture}
\end{minipage}}%
\hfill
\subfloat[Exp.2: Partial-weight exposure\label{fig:attack-Independence-leakage}]{%
\begin{minipage}[b]{0.48\columnwidth}
\centering
\begin{tikzpicture}
\begin{axis}[
    width=0.94\linewidth,
    height=3.5cm,
    xlabel={Query Budget},
    ylabel={Attack Accuracy},
    symbolic x coords={100,200,400,800,1600,3200,5000},
    ymin=0.1, ymax=0.95,
    legend style={at={(0.5,1.55)},anchor=north,legend columns=2,font=\small},
    grid=both,
    grid style={line width=.1pt, draw=gray!10},
    major grid style={line width=.2pt, draw=gray!50},
]
\addplot[mark=triangle*,green,thick] coordinates {
    (100, 0.11110000312328339)
    (200, 0.24639999866485596)
    (400, 0.26230001449584960)
    (800, 0.29859998822212220)
    (1600,0.35220000147819520)
    (3200,0.42429998517036440)
    (5000,0.46039998531341550)
};
\addlegendentry{100\%}

\addplot[mark=*,pink,thick] coordinates {
    (100, 0.6328999996185303)
    (200, 0.7644000053405762)
    (400, 0.7799999713897705)
    (800, 0.8064000010490417)
    (1600,0.8392999768257141)
    (3200,0.8748999834060669)
    (5000,0.8869000077247620)
};
\addlegendentry{75\%}

\addplot[mark=*,orange,thick] coordinates {
    (100, 0.6978999972343445)
    (200, 0.7591000199317932)
    (400, 0.7727000117301941)
    (800, 0.8219000101089478)
    (1600,0.8547000288963318)
    (3200,0.8756999969482422)
    (5000,0.8866000175476074)
};
\addlegendentry{50\%}

\addplot[mark=square*,red,thick] coordinates {
    (100, 0.6758000254631042)
    (200, 0.7463999986648560)
    (400, 0.8148999810218811)
    (800, 0.8488000035285950)
    (1600,0.8669000267982483)
    (3200,0.8866999745368958)
    (5000,0.9011999964714050)
};
\addlegendentry{25\%}
\end{axis}
\end{tikzpicture}
\end{minipage}}%

\par\vspace{0.6em}

\subfloat[Exp. 3: Naive representation\label{fig:attack-deep-rep}]{%
\begin{minipage}[b]{0.48\columnwidth}
\centering
\begin{tikzpicture}
\begin{axis}[
    width=0.94\linewidth,
    height=3.5cm,
    xlabel={Query Budget},
    ylabel={Attack Accuracy},
    symbolic x coords={100,200,400,800,1600,3200,5000},
    ymin=0.1, ymax=0.95,
    legend style={at={(0.3,1.45)},anchor=north,legend columns=2,font=\small},
    grid=both,
    grid style={line width=.1pt, draw=gray!10},
    major grid style={line width=.2pt, draw=gray!50},
]
\addplot[mark=triangle*,green,thick] coordinates {
    (100, 0.2867000102996826)
    (200, 0.2867)
    (400, 0.4228000044822693)
    (800, 0.491100013256073)
    (1600,0.46630001068115234)
    (3200,0.5519999861717224)
    (5000,0.5583999752998352)
};
\addlegendentry{\methodname}

\addplot[mark=square*,red,thick] coordinates {
    (100, 0.3231000006198883)
    (200, 0.5537999868392944)
    (400, 0.6855999827384949)
    (800, 0.7050999999046326)
    (1600,0.7013999819755554)
    (3200,0.7433000206947327)
    (5000,0.7781999707221985)
};
\addlegendentry{Deep Rep.}
\end{axis}
\end{tikzpicture}
\end{minipage}}%
\hfill
\subfloat[Exp. 4: Feature Richness\label{fig:attack-invariance-mi}]{%
\begin{minipage}[b]{0.48\columnwidth}
\centering
\makebox[\linewidth][c]{%
\begin{tikzpicture}
\begin{axis}[
    name=MIaxisB,
    width=0.94\linewidth,
    height=3.5cm,
    xlabel={MI},
    ymin=0.45, ymax=0.9,
    xtick={4, 6, 8, 10, 12},
    legend style={at={(0.5,1.65)},anchor=north,legend columns=2,font=\small},
    grid=both,
    grid style={line width=.1pt, draw=gray!10},
    major grid style={line width=.2pt, draw=gray!50},
]
\addplot[mark=square*,blue,thick] coordinates {
    (4.01080,0.7508)
    (6.92210,0.6816)
    (9.25430,0.5802)
    (11.7751,0.4911)
};
\addlegendentry{$|Q|$=800}

\addplot[mark=triangle*,green,thick] coordinates {
    (4.01080,0.8400)
    (6.92210,0.7849)
    (9.25430,0.7048)
    (11.7751,0.5520)
};
\addlegendentry{$|Q|$=3200}

\addplot[mark=triangle*,orange,thick] coordinates {
    (4.01080,0.8531)
    (6.92210,0.8323)
    (9.25430,0.7774)
    (11.7751,0.5584)
};
\addlegendentry{$|Q|$=5000}

\addlegendimage{mark=x,black,dashed,thick}
\addlegendentry{NPMI}
\end{axis}

\begin{axis}[
    at={(MIaxisB.south west)},
    anchor=south west,
    width=0.94\linewidth,
    height=3.5cm,
    xlabel={},
    ylabel={NPMI},
    ymin=0.2, ymax=1.0,
    xtick={4, 6, 8, 10, 12},
    ytick={0.2,0.6,1.0},
    yticklabels={0.2,0.6,1},
    axis x line=none,
    axis y line=right,
    ylabel style={
        at={(axis description cs:1.18,0.85)},
        anchor=south,
    },
]
\addplot[
    dashed,
    mark=x,
    black,
    thick,
] coordinates {
    (4.01080,0.30316865211344257)
    (6.92210,0.5232282155167202)
    (9.25430,0.699514724557054)
    (11.7751,0.8900571445848703)
};
\end{axis}
\end{tikzpicture}
}%
\end{minipage}}%

\caption{Extraction ablations. (a) Increasing public-component reuse produces no consistent improvement. (b) Exposing more private-model information raises surrogate accuracy; percentages denote protected neurons. (c) At similar representation dimensions, naive final-layer offloading is easier to imitate than \methodname. (d) Higher estimated representation richness corresponds to lower attack accuracy.}
\label{fig:security-ablations}
\end{figure}

%% file: data/metrics-dim.tex
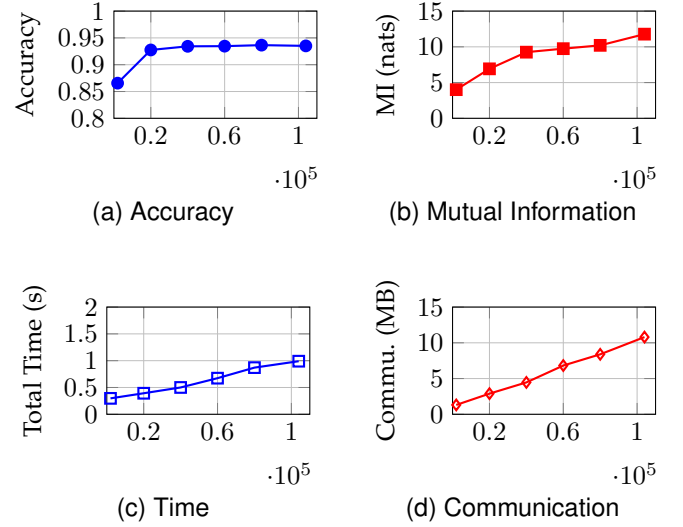
\begin{figure}[tb]
\centering
\subfloat[Accuracy\label{fig:acc-dim}]{%
\begin{minipage}[b]{0.48\columnwidth}
\centering
\begin{tikzpicture}
  \begin{axis}[
      ylabel={Accuracy},
      xmin=0, xmax=110000,
      ymin=0.8, ymax=1,
      xtick={20000,60000,100000},
      width=\linewidth,
      height=3cm,
      grid=both,
      grid style={line width=.1pt, draw=gray!10},
      major grid style={line width=.2pt, draw=gray!50},
  ]
  \addplot[mark=*,blue,thick] coordinates {
      (2000,0.8656)
      (20000,0.9275)
      (40000,0.9343)
      (60000,0.9346)
      (80000,0.9365)
      (104000,0.9351)
  };
  \end{axis}
\end{tikzpicture}
\end{minipage}}%
\hfill
\subfloat[Mutual Information\label{fig:mi-dim}]{%
\begin{minipage}[b]{0.48\columnwidth}
\centering
\begin{tikzpicture}
  \begin{axis}[
      ylabel={MI (nats)},
      xmin=0, xmax=110000,
      ymin=0, ymax=15,
      xtick={20000,60000,100000},
      width=\linewidth,
      height=3cm,
      grid=both,
      grid style={line width=.1pt, draw=gray!10},
      major grid style={line width=.2pt, draw=gray!50},
  ]
  \addplot[mark=square*,red,thick] coordinates {
      (2000,4.0108)
      (20000,6.9221)
      (40000,9.2543)
      (60000,9.7436)
      (80000,10.2082)
      (104000,11.7751)
  };
  \end{axis}
\end{tikzpicture}
\end{minipage}}%

\par\vspace{0.6em}

\subfloat[Time\label{fig:time-dim}]{%
\begin{minipage}[b]{0.48\columnwidth}
\centering
\begin{tikzpicture}
  \begin{axis}[
      ylabel={Total Time (s)},
      xmin=0, xmax=110000,
      ymin=0, ymax=2.0,
      xtick={20000,60000,100000},
      width=\linewidth,
      height=3cm,
      legend pos=north west,
      grid=both,
      grid style={line width=.1pt, draw=gray!10},
      major grid style={line width=.2pt, draw=gray!50},
  ]
  \addplot[mark=square,blue,thick] coordinates {
    (2000  ,0.2973)
    (20000 ,0.3925)
    (40000 ,0.5004)
    (60000 ,0.6738)
    (80000 ,0.8705)
    (104000,0.9903)
  };
  \end{axis}
\end{tikzpicture}
\end{minipage}}%
\hfill
\subfloat[Communication\label{fig:commu-dim}]{%
\begin{minipage}[b]{0.48\columnwidth}
\centering
\begin{tikzpicture}
  \begin{axis}[
      ylabel={Commu. (MB)},
      xmin=0, xmax=110000,
      ymin=0, ymax=15,
      xtick={20000,60000,100000},
      width=\linewidth,
      height=3cm,
      legend pos=north east,
      grid=both,
      grid style={line width=.1pt, draw=gray!10},
      major grid style={line width=.2pt, draw=gray!50},
  ]
  \addplot[mark=diamond,red,thick] coordinates {
    (2000  ,1.31)
    (20000 ,2.89)
    (40000 ,4.46)
    (60000 ,6.82)
    (80000 ,8.39)
    (104000,10.8)
  };
  \end{axis}
\end{tikzpicture}
\end{minipage}}%

\caption{Representation richness--cost tradeoff as total latent dimension increases. (a) Task accuracy, (b) estimated input--representation MI, (c) CPU end-to-end latency, and (d) communication. Accuracy stabilizes before estimated MI, while system cost grows with representation size.}
\label{fig:metrics-dim}
\end{figure}

%% file: data/client-cost.tex
\begin{figure}[tb]
\centering
\subfloat[Time\label{fig:client-time-dim}]{%
\begin{minipage}[b]{0.48\columnwidth}
\centering
\begin{tikzpicture}
  \begin{axis}[
      ylabel={Time (s)},
      xmin=0, xmax=110000,
      ymin=0, ymax=1,
      xtick={20000,60000,100000},
      width=0.96\linewidth,
      height=3cm,
      legend pos=north west,
      grid=both,
      grid style={line width=.1pt, draw=gray!10},
      major grid style={line width=.2pt, draw=gray!50},
  ]
  \addplot[mark=o,blue,thick] coordinates {
      (2000  ,0.0503)
      (20000 ,0.0945)
      (40000 ,0.1644)
      (60000 ,0.2558)
      (80000 ,0.3725)
      (104000,0.4903)
  };
  \end{axis}
\end{tikzpicture}
\end{minipage}}%
\hfill
\subfloat[Memory\label{fig:client-mem-dim}]{%
\begin{minipage}[b]{0.48\columnwidth}
\centering
\begin{tikzpicture}
  \begin{axis}[
      ylabel={Mem. (MiB)},
      xmin=0, xmax=110000,
      ymin=800, ymax=1500,
      xtick={20000,60000,100000},
      width=0.96\linewidth,
      height=3cm,
      legend pos=north east,
      grid=both,
      grid style={line width=.1pt, draw=gray!10},
      major grid style={line width=.2pt, draw=gray!50},
  ]
  \addplot[mark=triangle,red,thick] coordinates {
      (2000  ,940.0869)
      (20000 ,959.4424)
      (40000 ,1021.2615)
      (60000 ,1073.1465)
      (80000 ,1141.8411)
      (104000,1368.3979)
  };
  \end{axis}
\end{tikzpicture}
\end{minipage}}%
\caption{Client overhead of executing the public encoder locally at 32-bit precision as total latent dimension increases: (a) CPU time and (b) memory usage.}
\label{fig:client-cost-dim}
\end{figure}
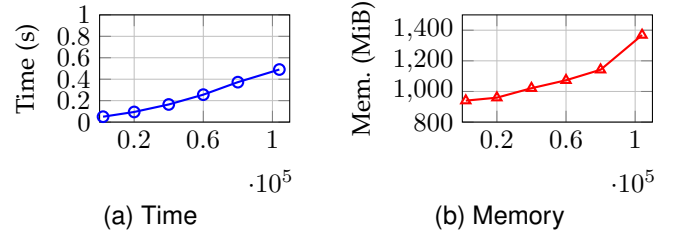

%% file: data/hfg.tex
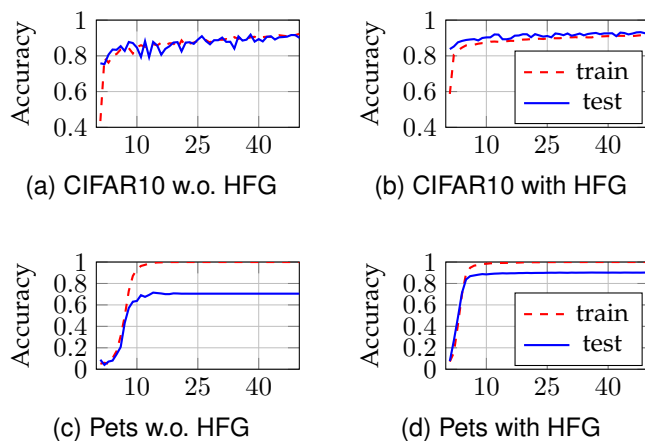
\begin{figure}[tb]
\centering
\subfloat[CIFAR10 w.o. HFG\label{fig:hfg-cifar10-non-hfg}]{%
\begin{minipage}[b]{0.48\columnwidth}
\centering
\begin{tikzpicture}
  \begin{axis}[
      ylabel={Accuracy},
      xmin=0, xmax=50,
      ymin=0.4, ymax=1,
      xtick={10, 25, 40},
      width=\linewidth,
      height=3cm,
      grid=both,
      legend pos=south east,
      grid style={line width=.1pt, draw=gray!10},
      major grid style={line width=.2pt, draw=gray!50},
  ]
  \addplot[dashed, red,thick] coordinates {
  (1, 0.4349)
(2, 0.7864)
(3, 0.7621)
(4, 0.8001)
(5, 0.8104)
(6, 0.8352)
(7, 0.8369)
(8, 0.8431)
(9, 0.8131)
(10, 0.849)
(11, 0.8594)
(12, 0.8498)
(13, 0.8571)
(14, 0.8657)
(15, 0.8585)
(16, 0.8733)
(17, 0.8604)
(18, 0.8699)
(19, 0.8699)
(20, 0.8433)
(21, 0.8865)
(22, 0.8579)
(23, 0.8688)
(24, 0.8741)
(25, 0.8838)
(26, 0.8853)
(27, 0.8882)
(28, 0.874)
(29, 0.8924)
(30, 0.8907)
(31, 0.8898)
(32, 0.894)
(33, 0.886)
(34, 0.9046)
(35, 0.8944)
(36, 0.9049)
(37, 0.8958)
(38, 0.9035)
(39, 0.9067)
(40, 0.8949)
(41, 0.9047)
(42, 0.9072)
(43, 0.9057)
(44, 0.91)
(45, 0.9155)
(46, 0.9124)
(47, 0.9098)
(48, 0.9101)
(49, 0.9164)
(50, 0.9226)
  };
  \addplot[blue,thick] coordinates {
 (1, 0.7582)
(2, 0.7533)
(3, 0.8066)
(4, 0.8352)
(5, 0.8336)
(6, 0.8542)
(7, 0.8275)
(8, 0.8774)
(9, 0.8715)
(10, 0.8462)
(11, 0.7936)
(12, 0.8819)
(13, 0.7925)
(14, 0.8778)
(15, 0.8578)
(16, 0.8063)
(17, 0.8391)
(18, 0.8535)
(19, 0.8743)
(20, 0.8766)
(21, 0.8252)
(22, 0.8741)
(23, 0.8761)
(24, 0.8671)
(25, 0.8905)
(26, 0.9099)
(27, 0.8826)
(28, 0.8646)
(29, 0.8875)
(30, 0.8874)
(31, 0.9022)
(32, 0.878)
(33, 0.8788)
(34, 0.8409)
(35, 0.9085)
(36, 0.8854)
(37, 0.8904)
(38, 0.9186)
(39, 0.8796)
(40, 0.8937)
(41, 0.8807)
(42, 0.9062)
(43, 0.9095)
(44, 0.8913)
(45, 0.9036)
(46, 0.9051)
(47, 0.9074)
(48, 0.9163)
(49, 0.9162)
(50, 0.901)
  };
  \end{axis}
\end{tikzpicture}
\end{minipage}}%
\hfill
\subfloat[CIFAR10 with HFG\label{fig:hfg-cifar10-hfg}]{%
\begin{minipage}[b]{0.48\columnwidth}
\centering
\begin{tikzpicture}
  \begin{axis}[
      ylabel={Accuracy},
      xmin=0, xmax=50,
      ymin=0.4, ymax=1,
      xtick={10, 25, 40},
      width=\linewidth,
      legend pos=south east,
      height=3cm,
      grid=both,
      grid style={line width=.1pt, draw=gray!10},
      major grid style={line width=.2pt, draw=gray!50},
  ]
  \addplot[dashed, red,thick] coordinates {
(1, 0.5868)
(2, 0.808)
(3, 0.8306)
(4, 0.846)
(5, 0.8583)
(6, 0.8608)
(7, 0.8622)
(8, 0.8679)
(9, 0.8711)
(10, 0.8708)
(11, 0.8789)
(12, 0.88)
(13, 0.8784)
(14, 0.8792)
(15, 0.8819)
(16, 0.8841)
(17, 0.8835)
(18, 0.8849)
(19, 0.8908)
(20, 0.8915)
(21, 0.8936)
(22, 0.8933)
(23, 0.8958)
(24, 0.8942)
(25, 0.8973)
(26, 0.8992)
(27, 0.8984)
(28, 0.8985)
(29, 0.9007)
(30, 0.904)
(31, 0.9035)
(32, 0.9045)
(33, 0.9008)
(34, 0.9049)
(35, 0.9031)
(36, 0.9033)
(37, 0.909)
(38, 0.9086)
(39, 0.9093)
(40, 0.9104)
(41, 0.9078)
(42, 0.912)
(43, 0.9121)
(44, 0.9119)
(45, 0.9132)
(46, 0.9104)
(47, 0.9128)
(48, 0.9151)
(49, 0.9157)
(50, 0.9154)
  };
  \addlegendentry{train}
  \addplot[blue,thick] coordinates {
(1, 0.8377)
(2, 0.8523)
(3, 0.8757)
(4, 0.8834)
(5, 0.8887)
(6, 0.8915)
(7, 0.8934)
(8, 0.8861)
(9, 0.8999)
(10, 0.9032)
(11, 0.9035)
(12, 0.9198)
(13, 0.9186)
(14, 0.8928)
(15, 0.8928)
(16, 0.897)
(17, 0.9128)
(18, 0.9139)
(19, 0.9209)
(20, 0.9206)
(21, 0.9041)
(22, 0.9215)
(23, 0.9144)
(24, 0.9076)
(25, 0.9109)
(26, 0.9037)
(27, 0.9251)
(28, 0.922)
(29, 0.9156)
(30, 0.9251)
(31, 0.9125)
(32, 0.9179)
(33, 0.9247)
(34, 0.9147)
(35, 0.9223)
(36, 0.9257)
(37, 0.9281)
(38, 0.9212)
(39, 0.9256)
(40, 0.9256)
(41, 0.9263)
(42, 0.9102)
(43, 0.93)
(44, 0.9276)
(45, 0.9199)
(46, 0.9247)
(47, 0.9312)
(48, 0.9324)
(49, 0.9269)
(50, 0.9267)
  };
  \addlegendentry{test}
  \end{axis}
\end{tikzpicture}
\end{minipage}}%

\par\vspace{0.6em}

\subfloat[Pets w.o. HFG\label{fig:hfg-pets-no-hfg}]{%
\begin{minipage}[b]{0.48\columnwidth}
\centering
\begin{tikzpicture}
  \begin{axis}[
      ylabel={Accuracy},
      xmin=0, xmax=50,
      ymin=0, ymax=1,
      xtick={10, 25, 40},
      width=\linewidth,
      height=3cm,
      legend pos=south east,
      grid=both,
      grid style={line width=.1pt, draw=gray!10},
      major grid style={line width=.2pt, draw=gray!50},
  ]
  \addplot[dashed, red,thick] coordinates {
(1, 0.0511)
(2, 0.0603)
(3, 0.0519)
(4, 0.1038)
(5, 0.1696)
(6, 0.3103)
(7, 0.4946)
(8, 0.7307)
(9, 0.8704)
(10, 0.9312)
(11, 0.9633)
(12, 0.9736)
(13, 0.9861)
(14, 0.9921)
(15, 0.9959)
(16, 0.997)
(17, 0.9978)
(18, 0.9995)
(19, 0.9997)
(20, 0.9997)
(21, 1.0)
(22, 1.0)
(23, 1.0)
(24, 1.0)
(25, 1.0)
(26, 1.0)
(27, 1.0)
(28, 1.0)
(29, 1.0)
(30, 1.0)
(31, 1.0)
(32, 1.0)
(33, 1.0)
(34, 1.0)
(35, 1.0)
(36, 1.0)
(37, 1.0)
(38, 1.0)
(39, 1.0)
(40, 1.0)
(41, 1.0)
(42, 1.0)
(43, 1.0)
(44, 1.0)
(45, 1.0)
(46, 1.0)
(47, 1.0)
(48, 1.0)
(49, 1.0)
(50, 1.0)
  };
  \addplot[blue,thick] coordinates {
(1, 0.0886)
(2, 0.0406)
(3, 0.0714)
(4, 0.0804)
(5, 0.1398)
(6, 0.2071)
(7, 0.4328)
(8, 0.5715)
(9, 0.6299)
(10, 0.637)
(11, 0.6907)
(12, 0.6751)
(13, 0.6931)
(14, 0.7146)
(15, 0.7105)
(16, 0.7067)
(17, 0.7005)
(18, 0.7007)
(19, 0.7075)
(20, 0.7067)
(21, 0.7046)
(22, 0.7048)
(23, 0.7046)
(24, 0.7046)
(25, 0.7046)
(26, 0.7048)
(27, 0.7048)
(28, 0.7048)
(29, 0.7048)
(30, 0.7048)
(31, 0.7048)
(32, 0.7048)
(33, 0.7048)
(34, 0.7048)
(35, 0.7048)
(36, 0.7048)
(37, 0.7048)
(38, 0.7048)
(39, 0.7048)
(40, 0.7048)
(41, 0.7048)
(42, 0.7048)
(43, 0.7048)
(44, 0.7048)
(45, 0.7048)
(46, 0.7048)
(47, 0.7048)
(48, 0.7048)
(49, 0.7048)
(50, 0.7048)
  };
  \end{axis}
\end{tikzpicture}
\end{minipage}}%
\hfill
\subfloat[Pets with HFG\label{fig:hfg-pets-hfg}]{%
\begin{minipage}[b]{0.48\columnwidth}
\centering
\begin{tikzpicture}
  \begin{axis}[
      ylabel={Accuracy},
      xmin=0, xmax=50,
      ymin=0, ymax=1.0,
      xtick={10, 25, 40},
      width=\linewidth,
      height=3cm,
      legend pos=south east,
      grid=both,
      grid style={line width=.1pt, draw=gray!10},
      major grid style={line width=.2pt, draw=gray!50},
  ]
  \addplot[dashed, red,thick] coordinates {
(1, 0.0712)
(2, 0.1549)
(3, 0.4155)
(4, 0.6989)
(5, 0.8916)
(6, 0.9462)
(7, 0.9663)
(8, 0.9728)
(9, 0.9812)
(10, 0.9837)
(11, 0.9878)
(12, 0.9899)
(13, 0.991)
(14, 0.9916)
(15, 0.9921)
(16, 0.9924)
(17, 0.994)
(18, 0.9948)
(19, 0.9967)
(20, 0.9978)
(21, 0.9976)
(22, 0.9981)
(23, 0.9978)
(24, 0.9989)
(25, 0.9992)
(26, 0.9992)
(27, 0.9992)
(28, 0.9992)
(29, 0.9995)
(30, 0.9997)
(31, 0.9997)
(32, 0.9997)
(33, 0.9997)
(34, 0.9997)
(35, 1.0)
(36, 1.0)
(37, 1.0)
(38, 1.0)
(39, 1.0)
(40, 1.0)
(41, 1.0)
(42, 1.0)
(43, 1.0)
(44, 1.0)
(45, 1.0)
(46, 1.0)
(47, 1.0)
(48, 1.0)
(49, 1.0)
(50, 1.0)
  };
  \addlegendentry{train}
  \addplot[blue,thick] coordinates {
(1, 0.0758)
(2, 0.2685)
(3, 0.4843)
(4, 0.713)
(5, 0.8373)
(6, 0.8681)
(7, 0.8749)
(8, 0.8817)
(9, 0.8874)
(10, 0.8847)
(11, 0.8885)
(12, 0.8915)
(13, 0.8929)
(14, 0.8951)
(15, 0.8937)
(16, 0.8951)
(17, 0.8951)
(18, 0.897)
(19, 0.8972)
(20, 0.8986)
(21, 0.8972)
(22, 0.8983)
(23, 0.8986)
(24, 0.8986)
(25, 0.8997)
(26, 0.9002)
(27, 0.8997)
(28, 0.9011)
(29, 0.9002)
(30, 0.8994)
(31, 0.9011)
(32, 0.9002)
(33, 0.9008)
(34, 0.9008)
(35, 0.9013)
(36, 0.9019)
(37, 0.9013)
(38, 0.9013)
(39, 0.9011)
(40, 0.9008)
(41, 0.9013)
(42, 0.9011)
(43, 0.9005)
(44, 0.9016)
(45, 0.9013)
(46, 0.9008)
(47, 0.9011)
(48, 0.9008)
(49, 0.9016)
(50, 0.9019)
  };
  \addlegendentry{test}
  \end{axis}
\end{tikzpicture}
\end{minipage}}%

\caption{HFG ablation on CIFAR10 and Pets over 50 training epochs. Dashed red curves show training accuracy and solid blue curves show test accuracy. HFG stabilizes convergence and reduces overfitting on the smaller Pets dataset.}
\label{fig:eval-hfg}
\end{figure}

%% file: sections/related.tex
\section{Related Work}

\noindent\textbf{Cryptographic Optimizations.} Pure-FHE inference protects the complete model~\cite{cryptoNets16icml}, including its architecture when circuit-private schemes are used~\cite{bourse2016FHECircuitPrivacy,gentry2009thesis}. Prior optimizations improve ciphertext packing and reduce rotations~\cite{brutzkusLowLatencyPrivacy2019,usenix22cheetah,kimOptimizedPrivacyPreservingCNN2023,singhHyenaBalancingPacking2024,he2024RhombusFastHomomorphic}, or approximate nonlinear activations for arithmetic FHE schemes~\cite{bfv12eprint,ckks17asiacrypt,zimermanConvertingTransformersPolynomial2024,nam2025SLOTHELazyApproximation,aoAutoFHEAutomatedAdaption2024a,diaaFastPrivateInference2024}. Hybrid protocols reduce latency but require depth-dependent interaction and high communication~\cite{xu2024PrivCirNetEfficientPrivate,xu2025BreakingLayerBarrier}; precomputation-based systems reduce online cost but retain substantial per-input work~\cite{ngGForceGPUFriendlyOblivious2021,feng2025PantherPracticalSecure,fu2025SwiftFastSecure,liu2025AntelopeFastSecure}. \methodname changes the system partition and remains compatible with these optimizations.

\noindent\textbf{Heuristic Approaches.} Other systems protect only a small predictor over pretrained features~\cite{brutzkusLowLatencyPrivacy2019,ruanPrivateEfficientAccurate2023,xia2026CryptPEFTEfficientPrivate}, selected neurons~\cite{renPrivDNNSecureMultiParty2024}, or selected layers~\cite{zhangNoPrivacyLeft2024,shen2022SOTERGuardingBlackbox}. Related acceleration techniques also expose model details without analyzing their leakage~\cite{chenTHEXPrivacyPreservingTransformer2022,yan2025CometAcceleratingPrivate,scope2025ccs}. As Section~\ref{sec:why-fail} shows, assuming that an exposed partial network is unusable can weaken model-side security; \methodname instead makes the public/private split explicit and evaluates its extraction risk.

%% file: sections/discon.tex
\section{Discussion and Conclusion}\label{sec:dis-con}

\methodname demonstrates the efficiency potential of treating an independently public encoder differently from a proprietary task-specific predictor. Its formalization identifies information-theoretic conditions that preserve uncertainty about proprietary functionality, and the extraction evaluation shows robustness comparable to ordinary black-box access when attackers reuse the exposed components. Together, these results support public-encoder offloading as a practical route to reducing encrypted computation while retaining model-side protection.

The evaluated system covers CNN-based classification encoders across image and audio applications. Extending \methodname to transformer and generative models introduces richer output spaces, potentially larger private predictors, and new challenges in assessing representation richness. These settings provide an important direction for future work. More broadly, efficient private inference requires both cryptographic optimization and an explicit separation between independently public resources and proprietary functionality.